\documentclass{article}

  \PassOptionsToPackage{numbers, compress}{natbib}

\usepackage[preprint]{neurips_2021}

\usepackage[utf8]{inputenc} 
\usepackage[T1]{fontenc}    
\usepackage{hyperref}       
\usepackage{url}            
\usepackage{booktabs}       
\usepackage{amsfonts}       
\usepackage{nicefrac}       
\usepackage{microtype}      
\usepackage{amsthm}
\usepackage{amssymb}
\usepackage{amsmath}
\usepackage{color}
\usepackage{algorithm}
\usepackage{algorithmicx}
\usepackage{algpseudocode}
\usepackage{graphicx}
\usepackage{listings}
\usepackage{placeins}
\usepackage{multirow}

\title{Can convolutional ResNets approximately preserve input distances? A frequency analysis perspective}

\author{%
  Lewis Smith\\
  OATML Group \\
  Department of Computer Science\\
  University of Oxford \\
  Oxford, United Kingdom \\
  \texttt{lsgs@robots.ox.ac.uk} \\
  \And
  Joost van Amersfoort\\
  OATML Group \\
  Department of Computer Science\\
  University of Oxford \\
  Oxford, United Kingdom \\
  \texttt{joost.van.amersfoort@cs.ox.ac.uk} \\
  \And
  Haiwen Huang \\
  OATML Group \\
  Department of Computer Science \\
  University of Oxford \\
  Oxford, United Kingdom \\
  \texttt{haiwen.huang2@cs.ox.ac.uk} \\
  \And
  Stephen Roberts \\
  Machine Learning Research Group \\
  Department of Engineering Science \\
  University of Oxford \\
  Oxford, United Kingdom \\
  \texttt{sjrob@robots.ox.ac.uk}
  \And
  Yarin Gal\\
  OATML Group \\
  Department of Computer Science\\
  University of Oxford \\
  Oxford, United Kingdom \\
  \texttt{yarin@cs.ox.ac.uk} \\
}

\newtheorem{lem}{Lemma}
\newtheorem{thm}{Theorem}
\DeclareMathOperator{\relu}{ReLU}
\DeclareMathOperator{\lip}{Lip}

\newcommand{\R}{\mathbb{R}}

\usepackage{todonotes}

\begin{document}
\maketitle

\begin{abstract}

	ResNets constrained to be bi-Lipschitz, that is, approximately distance preserving, have been a crucial component of recently proposed techniques for deterministic uncertainty quantification in neural models.
	We show that theoretical justifications for recent regularisation schemes trying to enforce such a constraint suffer from a crucial flaw -- the theoretical link between the regularisation scheme used and bi-Lipschitzness is only valid under conditions which do not hold in practice, rendering existing theory of limited use, despite the strong empirical performance of these models.
	We provide a theoretical explanation for the effectiveness of these regularisation schemes using a frequency analysis perspective, showing that under mild conditions these schemes will enforce a lower Lipschitz bound on the low-frequency projection of images. We then provide empirical evidence supporting our theoretical claims, and perform further experiments which demonstrate that our broader conclusions appear to hold when some of the mathematical assumptions of our proof are relaxed, corresponding to the setup used in prior work. In addition, we present a simple constructive algorithm to search for counter examples to the distance preservation condition, and discuss possible implications of our theory for future model design.

\end{abstract}

\section{Introduction}

A recent family of related methods have been proposed for obtaining SotA uncertainty estimates in deep learning in a single forward pass (i.e. not sampling or ensembling) \citep{van2020uncertainty,liu2020simple,van2021improving,mukhoti2021deterministic}.
The key feature these models share is the use of a neural mapping followed by a distance-sensitive output layer, such as a Gaussian process with a shift-invariant kernel.
These output layers revert to a default prior far away from the data; we would like to preserve this behaviour in neural network models.
However, achieving this with a standard neural network is not easy, as the neural part of the model is not constrained to preserve the distances in the original space.
In particular, two points $x$ and $y$ in the input space may be mapped by the feature mapping $f$ to any arbitrary locations in feature space - there is no relationship between the distance $||x - y||_2$ to the distance in the feature space $||f(x) - f(y)||$.
This means that using a layer which reverts to a uniform prediction away from a given point in the \textit{feature} space $f(x)$ does not guarantee any controls on the behaviour of uncertainty, as even points far away from the training data in input space may be mapped close to the feature representation of training points, and thus have high confidence under any method based on the feature space distances.
This is called `feature collapse'.

Previous work has resolved this problem by using a spectral regularisation scheme, in combination with a residual network structure, motivated by the following argument.
If we parameterise our feature mapping as $f(x) = x + g(x)$ , and then apply spectral regularisation \citep{miyato2018spectral,gouk2018regularisation} to the function $g$ such that its Lipschitz constant is less than 1, then $f$ is guaranteed to be \textit{bi-Lipschitz}, that is, to satisfy:
\begin{equation}
	L_1 d_{X}(x_{1}, x_{2}) \leq d_{Y}\left(f(x_{1}), f(x_{2})\right) \leq L_2 d_{X}(x_{1}, x_{2}),
	\label{eq:bilipschitz}
\end{equation}
for $0< L_1 < L_2$ \citep{behrmann2019invertible}.
A bi-Lipschitz condition bounds the distortion of distances in the input space.
This means that points sufficiently far apart in input space are guaranteed to be far apart in the feature space, while still allowing the model to learn a feature mapping which minimises the task loss.
As a bi-Lipschitz constraint would prevent feature collapse entirely, much of the single forward pass uncertainty work has adopted this approach of applying spectral normalisation to residual connections, often with this argument explicitly cited as motivation\citep{van2021improving,liu2020simple,mukhoti2021deterministic}.

To obtain scalability and high accuracy, however, two significant departures from this theory are made.
The first is to use dimensionality reduction.
To achieve good performance on image data, all prior work on this scheme for uncertainty includes layers which reduce the dimensionality of the input, such as strided convolutions.
However, as we show in Theorem \ref{thm:bilip}, if the input dimension is larger than the output dimension then the overall function \emph{cannot be bi-Lipschitz}.
Therefore, explanations based on the bi-Lipschitzness of the feature mapping, as put forward in prior work \citep{liu2020simple} are insufficient for explaining the strong empirical effectiveness of this scheme, as they do not apply to the model architectures actually used.
The second is the use of spectral normalisation with constants \textit{larger} than 1.
In this case there is no guarantee of bi-Lipschitzness.
However, most single pass uncertainty work finds that allowing the Lipschitz constant to be larger than 1 works better empirically, and still provides a benefit to the uncertainty quality.

Since we show that networks with downsampling layers cannot be bi-Lipschitz, we have to look for a different explanation for the strong performance of these models in practice.
As we explain in section \ref{section:decimation}, all downsampling layers in common use in convolution networks can be interpreted as distance-preserving operations followed by \emph{decimation}.
Decimation introduces feature collapse, or \emph{aliasing} in frequency content above the Nyquist rate of the decimated signal.
This suggests that we should analyse feature collapse in terms of the frequency content of images.

Since we know that the subsampling step introduces aliasing between images with frequency content above the Nyquist limit, it is natural to ask what effect a residual block has on the \emph{frequency content} of an image.
In section \ref{section:frequency_maths}, we establish mathematically that, under some relatively mild conditions, residual blocks must preserve the distance between \emph{low pass filtered} versions of their inputs.
This is relevant because the low frequency components of an image can be preserved under decimation, so if residual blocks preserve low-frequency distances, these distances will still be preserved after a spatial dimension reduction layer like a strided convolution.

Our contributions are as follows
\begin{enumerate}\itemsep0em
	\item We show that models including downsampling layers cannot be bi-Lipschitz, demonstrating a critical flaw with this justification of the regularisation scheme.
	\item We prove, under mild sufficient conditions, that we can establish a lower bound on the low-frequency distance between the feature maps of residual blocks given the low-frequency distance between their inputs. Since the low-frequency components of images will pass through a downsampling operation unchanged, this establishes a lower bound on the feature collapse even after downsampling.
	\item We describe a simple constructive algorithm for finding image pairs which exhibit feature collapse, that is, where $||x - y||$ is large but $f(x) \simeq f(y)$.
	\item We verify these theoretical claims empirically, investigate how they behave when the mathematical conditions of our proof are relaxed, and discuss the implications of these findings for future model design.
\end{enumerate}
\section{Theoretical Analysis}
\subsection{Why downsampling mappings cannot be bi-Lipschitz}
\label{section:no_bilip}

Firstly, we present a proof of why models which are not dimensionality preserving cannot be bi-Lipschitz, demonstrating that explaining the success of the regularisation scheme used by prior work needs further analysis.
Then, we motivate analysing this problem in terms of the frequency domain in section \ref{section:decimation}, briefly introducing necessary concepts, before outlining the argument behind our main theoretical claims in section \ref{section:frequency_maths}.
For brevity, we have focused on the outline of this argument and state results without proof in the main body; detailed proofs of all relevant theorems can be found in Appendix \ref{appendix:proofs}.

We first state one of our main claims - that dimensionality reducing models cannot be bi-Lipschitz.
We do not necessarily think that this observation is particularly novel mathematically, as it follows directly from two basic facts of analysis and topology; that Bi-Lipschitz mappings are homeomorphisms \citep{dructu2018geometric} and that $\R^n$ and $\R^m$ are not homeomorphic unless $m = n$ \citep{brouwer1911beweis}.
However, it appears to have been overlooked in the literature using this scheme for supervised models \citep{liu2020simple, van2020uncertainty, van2021improving, mukhoti2021deterministic}), and so we provide a straightforward proof of the following theorem
\begin{thm}
	\label{thm:bilip}
	Let $f$ be a function from $\mathbb{R}^n$ to $\mathbb{R}^m$, with $m < n$. Then $f$ is not bi-Lipschitz.
\end{thm}
This result is general, but our proof is non-constructive, so it is worth obtaining more intuition on why this must be the case by considering how to find counter-examples.
A constructive argument also suggests a constructive \emph{algorithm}, which we can use to find examples of feature collapse as a way to check our theoretical claims, which we consider in section \ref{section:counter_examples}.
If our function $f: \R^n \to \R^m$ is continuous and differentiable, then we can consider its local Jacobian $J(x) = \nabla_x f(x) \in \R^m \times \R^n$.
The value of the function around this point can be approximated by the Taylor series $f(x + t) = f(x) + J(x) t + \mathcal{O}(||t||^2)$.
If we consider the singular value decomposition of the Jacobian $J = U S V^\top$, there will only be $m$ columns of $V$ (the right singular vectors) corresponding to non-zero singular values.
The remaining $n - m$ columns of $V$ form a basis for the null space of $J$, that is, the set of vectors $\{ v \mid J v = 0\}$.
Therefore, starting at any point $x$, we can take a step $t \in \{ v \mid J(x) v = 0\}$ in any direction in this null space and have that $f(x) - f(x + t) = \mathcal{O}(||t||^2)$ from the Taylor expansion, since $J t = 0$.
We can repeat this process iteratively with $n$ steps $t_i$ of length $\epsilon$, such that each $t_i$ lies in the null space of the Jacobian $J(x + \sum_1^{i-1} t_i)$ to generate, starting from a point $x_0$, a counterpart $x_n = x_0 + \sum_{i=1}^n t_i$.
We can freely choose $t_i$ such that $\langle t_i, t_j \rangle \ge 0$ for all $i, j$, so that the steps do not cancel each other out.
In this case, the total distance covered by the trajectory, $|| x_0 - x_n|| = || \sum_{i=1}^n t_i|| \simeq \mathcal{O}(\epsilon n) $, whereas the corresponding change in feature space is $||f(x_0) - f(x_n)|| \simeq \mathcal{O}(n \epsilon ^2)$.
It is clear from this that by adding more steps to this trajectory and making $\epsilon$ arbitrarily small while holding $n \epsilon = C$ constant, the effect on the output can be made arbitrarily small ($\mathcal{O}(C \epsilon)$) while the length of the trajectory remains $\mathcal{O}(C)$, so we obtain a recipe for finding two points $x_0$ and $x_n$ with an arbitrarily small distance in feature space.
This argument is a little less mathematically precise than our proof of Theorem \ref{thm:bilip}, but it does suggest a constructive algorithm for finding \emph{examples} of feature collapse.
We discuss how to make this algorithm practical in section \ref{section:counter_examples}.

\subsection{Downsampling and feature collapse in convolutional networks}
\label{section:decimation}
Theorem \ref{thm:bilip} straightforwardly applies to models which treat the input as a vector in an unstructured space.
However, we are interested in particular in convolutional models, where the feature maps and downsampling operations have a spatial structure.
In particular, as observed in \citet{zhang2019making}, downsampling operations in convolutional networks - including average pooling, strided convolutions, and max pooling - can be thought of as a dimensionality preserving dense operation followed by an image decimation step, where we simply keep every $n^{th}$ sample from the output.
For instance, strided convolution, which is non-dimensionality preserving and thus non-invertible (and of particular interest since this is the downsampling operation most commonly used in modern ResNets) is exactly identical to a standard convolution followed by an image subsampling step where we drop pixels from the output; in other words, the following two Python snippets should give identical values for \texttt{y};
%
%
\begin{minipage}{0.4\textwidth}
	\begin{lstlisting}[language=Python]

y = conv2d(x, k, stride=2)

\end{lstlisting}
\end{minipage}
\vline\hspace{1em}
\begin{minipage}{0.48\textwidth}
	\begin{lstlisting}[language=Python]
def decimate(x, s):
    return x[..., ::s, ::s]

y = decimate(conv2d(x, k, stride=1), 2)
\end{lstlisting}
\end{minipage}\\
This lets us think of a ResNet as consisting entirely of residual blocks (which can be regularised to be bi-Lipschitz mappings) and a \textit{non-learnable} decimation operation.
\begin{figure}[h]
	\centering
	\includegraphics[width=.6\textwidth]{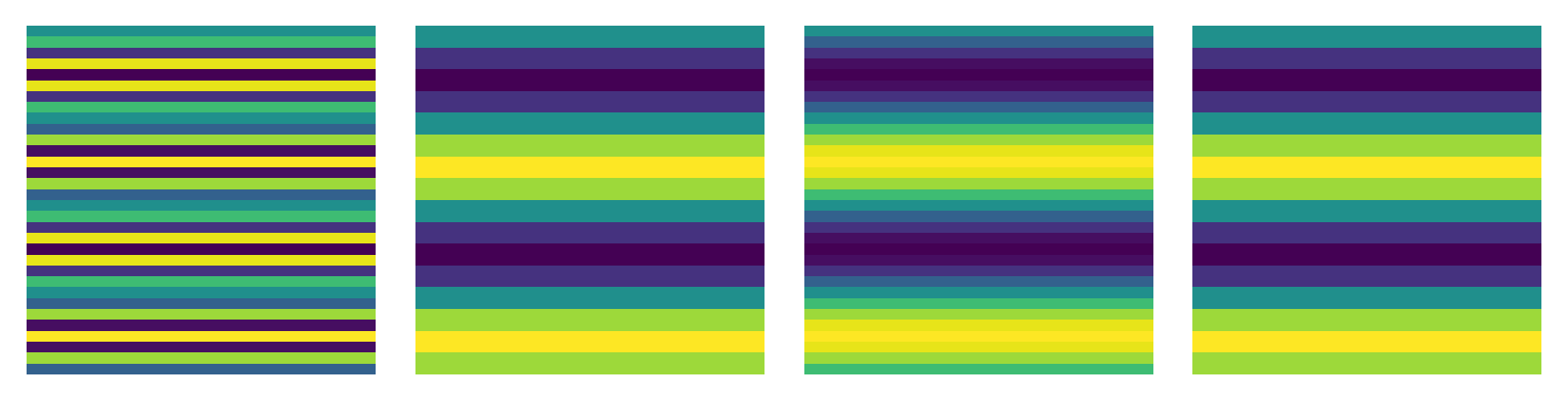}
	\caption{An example of two Fourier modes which are aliases of each other. From left to right; a Fourier basis function above the Nyquist rate, the same image decimated by a factor 2, a basis function below the Nyquist rate, and the second mode downsampled by a factor 2.
		After decimation, these modes become indistinguishable. Note that the image above the Nyquist frequency is aliased into a much lower frequency, while the second mode is still resolved correctly - using BlurPool avoids aliasing, instead mapping all frequencies above the Nyquist frequency to (approximately) zero.
	}
	\label{fig:alias_example}
\end{figure}
The effects of decimation on a signal are well understood. In particular, sufficient conditions for this downsampling operation to produce \emph{no} feature collapse/aliasing are known.
The Nyquist-Shannon theorem states that, for a given sampling frequency $u_s$, we can perfectly reconstruct a signal assuming it contains no frequencies above the Nyquist rate $u_s / 2$.
See Appendix \ref{section:nyquistshannon} for a more detailed discussion of decimation and aliasing.

This means that if we have two signals (such as feature maps in a neural network) $x$ and $y$ which differ in their content \emph{below the Nyquist rate}, then after decimation they will still be distinguishable, since their low-frequency content will be preserved by this operation.
Feature collapse between $x$ and $y$ will only be induced by decimation if the difference between $x$ and $y$ is all above the Nyquist frequency, as this will be removed by the filtering operation we apply before decimation.
%
%
%
\subsection{Residual networks and the frequency content of images}
\label{section:frequency_maths}

As discussed above, all downsampling operations in ResNets can be understood as non-strided operations followed by decimation.
The effects of decimation, combined with anti-aliasing, in the frequency domain are well understood - frequencies below the Nyquist will be unaffected, while those above will be removed from the signal.
But in order to establish anything about the implications of this, we need to connect the frequency content of the input of a residual block to the frequency content of its output.
This section addresses this problem, and uses it to prove our main theoretical claims about the effect of spectral normalisation in (convolutional) residual networks.

Here, by `residual block', we mean a function of the form $f(x) = x + g(x)$, where the residual connection $g = g_1 \circ g_2 \circ ... g_n$, and the components $g_i$ are all convolutions, ReLUs or batch normalisation layers.
We use the following notation; $x * y = (x * y)(t) = \int x(\tau) y(t - \tau) d\tau $ denotes the convolution of the functions $x$ and $y$. A capital letter represents the Fourier transform of a signal, so $X$ denotes the Fourier transform of the signal $x$. $H_u$ is an (ideal) low-pass filter, removing all frequencies in a signal above the cutoff frequency $u$. $x \cdot y$ denotes pointwise multiplication.

Note that the following discussion, which is based on the Fourier domain, considers only a single channel.
However, this is acceptable because downsampling and nonlinearities also act independently on the channels, and so we can consider the channels independently.
Convolutions may mix channels, but, as explained below, they cannot mix frequencies, and we can bound the Lipschitz constant of multi-channel images effectively.

Residual networks consist of identity mappings, addition, convolutions, nonlinearities (ReLUs) and batch normalisation.
Identity mappings have no effect on the image, so cannot change the frequency content.
Pointwise addition is linear, so will simply add the frequency content of the two images.
\textit{Convolutions} clearly can modify the frequency content of images, but in limited ways.
Since convolutions are diagonalised by the Fourier transform, they act in the frequency domain as independent pointwise multiplication of the frequency components of the input signal with the corresponding frequency components of the convolution kernel $K$.
Batch normalisation multiplies each channel by a scalar and adds a bias; this clearly can not change the relative scale of the frequencies (other than the 0 frequency term, which is affected by the bias).

The only operation left to deal with is the nonlinearity.
ReLUs are a nonlinear filter, and so they \emph{can} modify the frequency content in a more complex way.
In fact, a ReLU acts as a \textit{convolution} in the frequency domain.
To see this, consider an image $x$, and let $m = I[x > 0]$ be the binary mask of the locations in the image where the ReLU is active. Then we can write the ReLU as a pointwise multiplication, $\relu(x) = x \cdot m(x)$.
Because, by the convolution theorem, pointwise multiplication in the spatial domain must be convolution in the frequency domain, this means that the action of the ReLU on $X$ will be a convolution of $X$ with the Fourier transform $M(x)$ of the image mask, so $\relu(X) = M(x) * X$ (see figure \ref{fig:fourier_relu}).

\begin{figure}[t]
	\centering
	\includegraphics[width=0.8\textwidth]{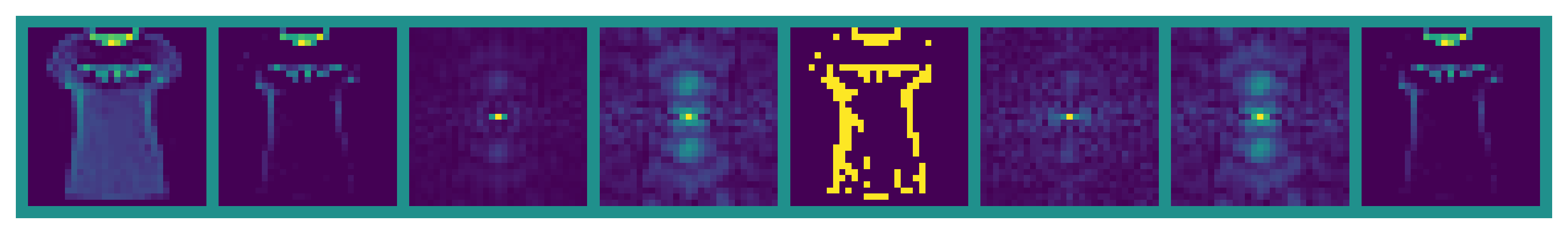}
	\caption{An image demonstrating the frequency domain representation of the ReLU, on an image from FashionMNIST.
		From left to right; a) the original image $x$, normalised to have zero mean and unit std. dev. so that the
		ReLU has a non-trivial effect. b) $z = \max(x,0)$ c) The FT of $x$, $X$. d) $Z$, calculated in the spatial domain. e) the mask $m = I[x > 0]$. f) The spectrum of the mask, $M$. g) $M * F(X)$, calculated by direct convolution in the frequency domain. Observe this is identical to image c. h) the inverse Fourier transform of image g, which again recovers b. We plot only the magnitude of the complex-valued Fourier spectra.
	}
	\label{fig:fourier_relu}
\end{figure}

If we place some sufficient conditions on the low frequency content of an image, we already have enough to show that residual connections will preserve low frequency distances.
First, we establish the following simple lemma about the convolution of distributions.
We can easily show that if a particular interval `dominates' a measure, then the convolution will reduce the amount it dominates by, which fits the intuition that convolution is a smoothing operation.
In this paper, we say an interval dominates a measure if there is more mass in that interval than in any other interval of the same size.
\begin{lem}
	Let $f$ and $g$ be (normalised) measures.
	Let an interval of length $L$ starting at $x$ be a `dominant' interval in $f$; that is, if $\int_x^{x+L} f(t) dt = C$, then $\int_y^{y+L} f(t) dt \le C \forall y$ if [x, x + L] is a dominant interval.
	Then the mass in that dominant interval is reduced (or stays the same) when $f$ is convolved with $g$, whatever $g$ is; that is $\int_x^{x+L} [f * g](t) dt \le C$.
	\label{lem:conv_concentration}
\end{lem}
Note that although Lemma \ref{lem:conv_concentration} is stated in terms of normalised distributions for simplicity, it generalises easily to cases where $g$ is an unnormalised measure.

We can now state a key supporting result.
\begin{lem}
	Let $H_u$ be defined as above, let $x, y$ be images, and let $v = x - y$. Assume that the difference image, $v$, is dominated (in the sense of Lemma \ref{lem:conv_concentration}) by frequencies below a cutoff frequency $u$.
	Then $||H_u(\relu(x) - \relu(y))|| < ||H_u(x) - H_u(y)||$.
	\label{lem:relu_contraction}
\end{lem}
Now, we have established that the ReLU will act (under certain circumstances) as a contraction on the low frequency component $X_u$ of its input.
Using the fact that most other network components act on frequencies in a very simple way, as discussed, it is then straightforward to extend this result to the whole residual connection, assuming we apply spectral normalisation to the convolutions:
\begin{thm}
	Let $f = x + g(x)$ be a convolutional residual block (i.e $g$ is a series of convolutions, batch normalisation and ReLUs), and assume that $g$ is regularised to be contraction ($\lip(g) < 1$), and that the conditions of Lemma \ref{lem:relu_contraction} hold on the input to the ReLU (i.e the difference image $x - y$ is low-frequency dominant).    Then, the low-passed distances between the output are lower-bounded by the low-passed distances on the input, that is $L ||H_u(x) - H_u(y)|| \le ||H_u(f(x)) - H_u(f(y))||$ for some constant $L > 0$.
	\label{thm:low_pass_bilip}
\end{thm}
This result can be extended easily to the whole network by the multiplicativity of (upper and lower) Lipschitz constants - if $\lip(f_1) = L_1$ and $\lip(f_2) = L_2$, then $\lip(f_1 \circ f_2) = L_1 L_2$.

The importance of this result is, as explained in the previous section, decimation only acts on the high frequency components of a signal.
To be specific, say that the highest frequency representable in a given signal is $u_{\textrm{max}}$. After decimation by a factor $D$, any signals above $u_{\textrm{max}} / D$ will be lost (if we filter before downsampling to avoid aliasing) and those below will be preserved.
Therefore, if one of the residual blocks is a strided convolution with stride $D$, essentially we can consider it as a normal residual block combined with the low-pass projection operation $H_{u_{\textrm{max}} / D}$.
Therefore, if Theorem \ref{thm:low_pass_bilip} holds on a pair of inputs $x, y$, then this lower bound holds equally well \emph{after decimation}, since $H_{u_{\textrm{max}} / D}(f(x))$ is unaffected by downsampling.

The relevant `low-frequency interval' for a given network, then,  is defined by the architecture - a ResNet containing 2 layers with stride 2, (and an arbitrary number of stride 1 layers) for example, has a total downsampling factor of 4 - that is, the final feature map can only resolve frequencies of $u_{\textrm{max}} / 4$ by Nyquist-Shannon.
Though the results above can apply to any interval which satisfies the assumptions, this defines the `low-frequency interval' of interest in our argument; in other words, we choose the cutoff frequency $u = u_{\textrm{max}} / D$ such that the operator $H_u$
is the identity on the final feature map, so we are interested in whether frequency content which can be resolved by the sampling rate of the final feature map is preserved.

It is important to mention some limitations, however.
Firstly, the theorem only establishes a bound if the \emph{difference} between images is low-band dominant.
If $x-y$ is dominated by high frequencies, then our theorem does not apply.
Secondly, the extension to the full network assumes that, not only is $x- y$ low band dominant, but $f_l(x) - f_l(y)$ is low-band dominant for all intermediate feature maps $l$ before downsampling layers.
We have not proved that this is necessarily the case even assuming it holds for the input.
In fact, we do not think it is possible to prove it without making further assumptions or imposing additional constraints on the network architecture based on our empirical results, discussed below.

Having said that, there are reasons to suppose that the result of Theorem $\ref{thm:low_pass_bilip}$ may hold even if our assumptions are mildly violated.
This is because the assumption of $x- y$ being dominated by its low-frequency component in the sense of lemma \ref{lem:conv_concentration} is a fairly strong assumption.
It is reasonable to suspect it may be a fairly \emph{weak} sufficient condition; it does not make any assumptions about the structure of the ReLU itself, other than that acts as convolution in the frequency domain.
This avoids dealing with the ReLU's input dependence, but is fairly crude, and so we might reasonably expect our main \emph{conclusion}, that residual networks will often approximately preserve distances in the low-frequency components of their input, to hold more broadly than our sufficient conditions.
This can only really be answered by an empirical study of standard neural networks trained on image datasets, and we investigate this empirically below.
\section{Finding Counter-Examples}
\label{section:counter_examples}
The proof in the previous section shows that the distances between two inputs $x$ and $y$ will be approximately preserved by a neural mapping, assuming that our assumptions hold.
These conditions seem reasonable, but we may be able to find counter examples.

We present a simple procedure to find such examples in Algorithm \ref{alg:find_counter_examples}.
\begin{algorithm}
	\caption{Given image $x_0$, find an image $x$ such that $f(x_0) \simeq f(x)$ and $||x_0 - x|| > 0$.}
	\label{alg:find_counter_examples}
	\begin{algorithmic}[0]
		\State \textbf{Input:} Input $x_0$, function $f$, step size $\eta$, number of steps $n$, distribution over unit vectors $D$
		\State $y_0 \gets f(x_0)$
		\For{ $i \in 0..n-1$ }
		\State $v \gets D()$ \Comment{Sample random step direction}
		\State $g \gets \nabla_x \left( ||f(x) - y_0|| \right) $  \Comment{find the gradient of the distance from the target output}
		\State $u \gets v - \frac{\langle v, g \rangle}{\langle g, g \rangle} g$ \Comment{Project the step to be orthogonal to the gradient}
		\State $x \gets x + \eta u$ \Comment{Take a step along the level set}
		\EndFor \\
		\Return $x$
	\end{algorithmic}
\end{algorithm}
The intuition behind this is identical to the constructive argument in Section \ref{section:no_bilip}.
The main problem with directly adapting that construction into a practical algorithm is that it requires computing the local Jacobian and its SVD, which in large feature spaces would be a severe bottleneck.
As an alternative, we consider taking steps which hold the norm of $||f(x) - y_0||$ constant, rather than the function $f(x)$ directly, as this only requires computing a vector gradient, exploiting the fact that the level sets of a function are locally orthogonal to the gradient, to achieve a similar aim of choosing steps that keep $f(x)$ approximately constant.
In practice, we need to add an initial small perturbation to $x$ because if $||f(x) - y_0||$ is zero then its gradient is also zero and the projection becomes unstable numerically.

We can influence the properties of the generated example $x$ by changing the distribution of $v$.
It is simple to, for example, choose $v$ to be band limited, to generate examples where $x_0 - x$ is either low-frequency or high-frequency dominant.
We investigate both of these below.
In practice, this algorithm does not totally prevent the distance from increasing due to the accumulation of errors from the finite step size.
There is an obvious trade off between the size of the perturbation generated, the tolerated error in keeping the output constant, and the number of iterations this algorithm is run for.

\section{Experiments}
\subsection{Verifying our theoretical analysis}
There are two key claims in Section \ref{section:frequency_maths} we want to verify empirically.
These are 1) Do the conditions of Lemma \ref{lem:conv_concentration}  hold (i.e are the images and feature maps low frequency dominant in the sense of this lemma)? And 2) do the \emph{conclusions} of Theorem \ref{thm:low_pass_bilip} hold, that is, is $||H_u(g(x)) - H_u(g(y))|| < ||H_u(x) - H_u(y)||$ for the inputs to each residual block?

Note that the first is a sufficient condition for Theorem \ref{thm:low_pass_bilip} (and thus the second point) to hold, but it is not a necessary one, and this may hold empirically on tested inputs even if domination does not.
In addition, we are interested in seeing whether we can find violations of our theorem when we relax the conditions of the proof.
In particular, in the proof we assume that the convolutions are strict contractions ($\lip(g) < 1$), whereas in previous work on distance aware learning, typically spectral norm with a coefficient between 3 and 6 has been used \citep{van2020uncertainty,liu2020simple,mukhoti2021deterministic}.
Our theorem applies to any \emph{possible} setting of the weights of a network - the \emph{specific} trained weights of a model may obey our conditions, even if the regularisation does not strictly enforce it, in which case our frequency-based explanation could still illuminate why these models do not appear to suffer from feature collapse.

In order to check these conditions, we train Wide ResNets \citep{zagoruyko2016wide} (with minor changes discussed in Appendix \ref{appendix:architectural_details}) on MNIST \citep{lecun1998gradient}, FashionMNIST \citep{xiao2017fashion} and CIFAR10 \citep{krizhevsky2009learning}.
We can check whether both of these conditions hold empirically, on data from these datasets and the intermediate feature maps of the residual network.
We can check the domination of an input image by direct numerical check on its Fourier transform.
In order to check the second condition, we consider a mini-batch of data $\{x_i \mid i \in 1..m\}$.
We can check empirically that the residual connection is as contraction on the low-frequency component by considering pair-wise distances before and after the residual connection, and calculating the proportion of the batch for which $||g(H_u(x_i)) - g(H_u(x_j))|| < ||H_u(x_i) - H_u(x_j)||$.
This test does not necessarily prove that the theorem holds for \emph{all} possible inputs, but if the theorem is violated widely we might expect to see violations on a particular dataset.
More importantly, if we cannot find violations on natural datasets, this does suggest that low-frequency distance is preserved \emph{between natural images}.
These results are shown for FashionMNIST in table \ref{table:theorem_checker_fmnist_only}.
We also carried out the same experiment on MNIST and CIFAR10, but we defer these results to Appendix \ref{appendix:more_exps} for brevity, as the results are consistent across these datasets.

As this table shows, Theorem \ref{thm:low_pass_bilip} is supported remarkably well by the evidence - even
though we cannot prove that it holds on all possible inputs, we are unable to find \emph{any} counterexamples in the image datasets tested.
We do find that the low-frequency domination condition holds on a reasonable proportion of images, but not enough to fully explain the adherence to the Theorem \ref{thm:low_pass_bilip}, though as mentioned, there are theoretical reasons to expect this sufficient condition to be loose.

In addition, our theorems only apply when the value of the Lipschitz constant is below 1.
As we might expect, when we increase the value of the spectral normalisation above 1, relaxing this constraint, we start to see violations of our condition in trained models.
However, this is fairly rare in practice, as Table \ref{table:theorem_checker_fmnist_only} shows.
Interestingly, we observe that models appear to converge to satisfy our condition during training; an example of this dynamic is shown in Figure \ref{fig:theorem_check_training_small}.
Models with residual Lipschitz constant less than 1 obey our theorem throughout training, while those with larger Lipschitz constant violate it early it training, but generally converge to mostly obey it on the training data.
We do not have a theoretical explanation for why this should happen as a result of training, but it is agreement with the empirical observation in prior work that  distance aware learning with spectral normalisation coefficients above 1 have performed well, despite their lack of guarantees \citep{van2020uncertainty,liu2020simple,van2021improving,mukhoti2021deterministic}.
\begin{figure}[t]
	\centering
	\label{fig:theorem_check_training_small}
	\includegraphics[width=\textwidth]{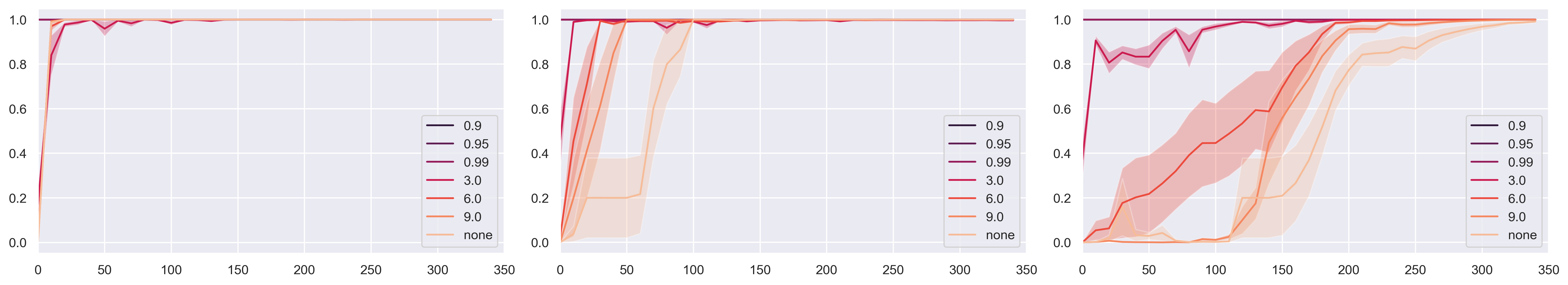}
	\caption{
		Evolution of the mean value of the theorem check during training for the first three blocks of a WideResNet trained on CIFAR10, for various values of the spectral normalisation coefficient ($\lip(g)$).
		The behaviour of other blocks is similar - these are given in Figure \ref{fig:theorem_check_training_full} (Appendix \ref{appendix:more_exps}).
		We plot the mean and standard error across 5 seeds.
	}
\end{figure}
\begin{table}[t]
	\centering
	\caption{Results of empirically checking the conditions of Lemma \ref{lem:conv_concentration} (low frequency
		domination) and Theorem \ref{thm:low_pass_bilip} (low frequency contraction).
		Numbers are the proportion of the dataset for which we found the condition to hold exactly; so 1 means no violations were found, 0.5 means it was true for half the inputs tested, etc.
		We report means and standard error over 25 seeds and use a WideResNet with depth 10 and widen factor of 1 for these datasets.
	}
	\label{table:theorem_checker_fmnist_only}
	\resizebox{0.9\columnwidth}{!}{%
		\begin{tabular}{@{}lllccccc@{}}
			\toprule
			                                                  & Dataset                 & $\lip(g)$        & $x$              & $f_1(x)$         & $f_2(x)$         & $f_3(x)$         & $f_4(x)$         \\ \midrule
			\multirow{7}{*}{Lemma \ref{lem:conv_concentration}}
			                                                  & \multirow{7}{*}{FMNIST}
			                                                  & 0.9                     & $0.86  \pm 0.00$ & $0.18  \pm 0.04$ & $0.24  \pm 0.04$ & $0.48  \pm 0.04$ & $0.47  \pm 0.03$                    \\
			                                                  &                         & 0.95             & $0.86  \pm 0.00$ & $0.18  \pm 0.04$ & $0.25  \pm 0.02$ & $0.48  \pm 0.03$ & $0.46  \pm 0.01$ \\
			                                                  &                         & 0.99             & $0.86  \pm 0.00$ & $0.18  \pm 0.05$ & $0.25  \pm 0.04$ & $0.51  \pm 0.02$ & $0.48  \pm 0.01$ \\
			                                                  &                         & 3.0              & $0.86  \pm 0.00$ & $0.24  \pm 0.03$ & $0.28  \pm 0.04$ & $0.53  \pm 0.04$ & $0.47  \pm 0.02$ \\
			                                                  &                         & 6.0              & $0.86  \pm 0.00$ & $0.24  \pm 0.03$ & $0.28  \pm 0.02$ & $0.53  \pm 0.02$ & $0.47  \pm 0.02$ \\
			                                                  &                         & 9.0              & $0.86  \pm 0.00$ & $0.25  \pm 0.01$ & $0.28  \pm 0.04$ & $0.54  \pm 0.04$ & $0.47  \pm 0.01$ \\
			                                                  &                         & no               & $0.86  \pm 0.00$ & $0.24  \pm 0.03$ & $0.25  \pm 0.03$ & $0.53  \pm 0.03$ & $0.47  \pm 0.01$ \\
			\midrule
			\multirow{7}{*}{Theorem \ref{thm:low_pass_bilip}} & \multirow{7}{*}{FMNIST} & 0.9              & n/a              & $1.00 \pm 0.00$  & $1.00 \pm 0.00$  & $1.00 \pm 0.00$  & $1.00 \pm 0.00$  \\
			                                                  &                         & 0.95             & n/a              & $1.00 \pm 0.00$  & $1.00 \pm 0.00$  & $1.00 \pm 0.00$  & $1.00 \pm 0.00$  \\
			                                                  &                         & 0.99             & n/a              & $1.00 \pm 0.00$  & $1.00 \pm 0.00$  & $1.00 \pm 0.00$  & $1.00 \pm 0.00$  \\
			                                                  &                         & 3.0              & n/a              & $0.98 \pm 0.04$  & $1.00 \pm 0.00$  & $1.00 \pm 0.00$  & $1.00 \pm 0.00$  \\
			                                                  &                         & 6.0              & n/a              & $0.95 \pm 0.11$  & $0.92 \pm 0.04$  & $1.00 \pm 0.00$  & $1.00 \pm 0.00$  \\
			                                                  &                         & 9.0              & n/a              & $1.00 \pm 0.01$  & $0.85 \pm 0.12$  & $1.00 \pm 0.00$  & $1.00 \pm 0.00$  \\
			                                                  &                         & no               & n/a              & $1.00 \pm 0.00$  & $0.86 \pm 0.07$  & $1.00 \pm 0.00$  & $1.00 \pm 0.00$  \\ \bottomrule
		\end{tabular}
	}
\end{table}
\subsection{Artificially finding counter-examples}
\begin{figure}
	\centering
	\label{fig:adv_trajectory}
	\includegraphics[width=\textwidth]{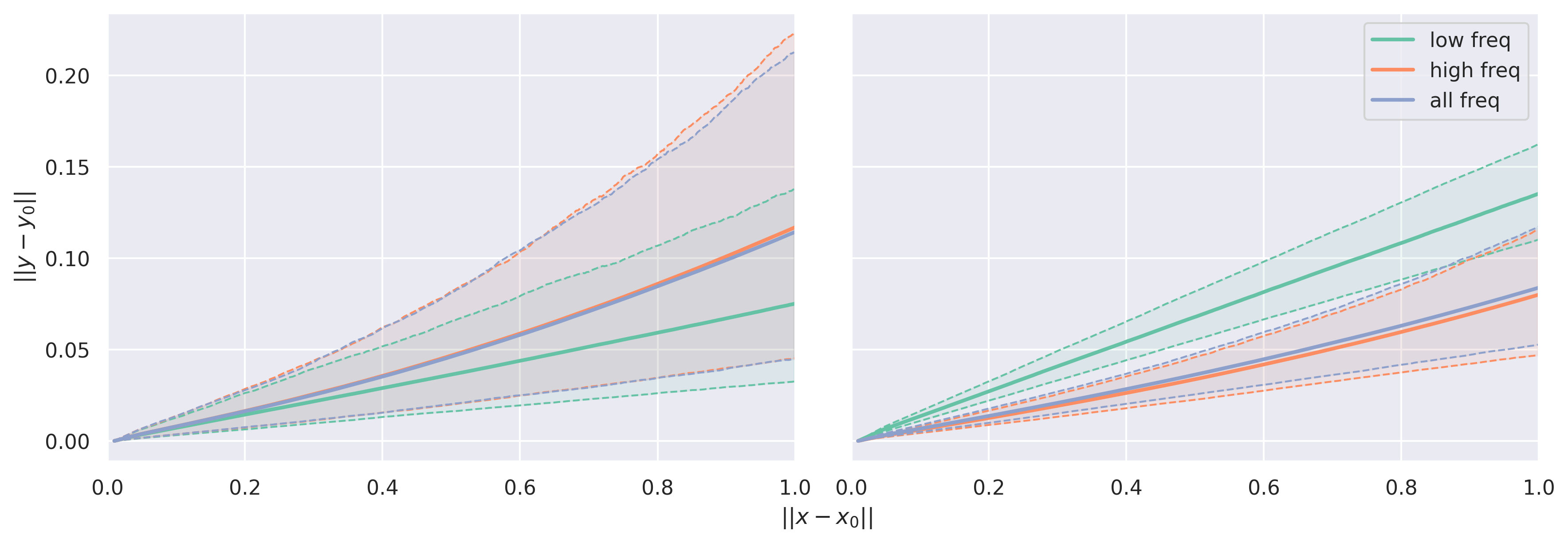}
	\caption{Distance in feature space ($||y - y_0||$) against distance in image space ($||x - x_0||$) during a run of Algorithm \ref{alg:find_counter_examples}. Left; for a standard model, right; for a model where we low-pass-filter the feature maps between residual blocks, enforcing the domination assumption.
		We show results where the trajectory is restricted to either low band or high band perturbation and unrestricted.
		On each, we show the mean and 10th and 90th quantile across 1000 images.
		We detail a more quantitative version of this experiment in appendix \ref{appendix:more_exps}.
	}
\end{figure}
The previous section showed that our results are descriptive of network behaviour on natural images.
Here, we describe testing our assumptions using Algorithm \ref{alg:find_counter_examples} to find counter-examples.
While the results of the previous section suggest that examples of feature collapse are rare on natural image datasets, we can find them by explicit optimisation, as shown in Figure \ref{fig:adv_trajectory}.

Note that this does \emph{not} disprove Theorem \ref{thm:low_pass_bilip} - an important assumption of our theorem is that not just $x - y$ is low-band dominant, but $f_l(x) - f_l(y)$ for all intermediate layers $l$.
As mentioned above, we have no proof that $f_l(x) - f_l(y)$ has to be low-band dominant if the input is (and indeed, this result shows no such proof could be found).
While the results of the previous section suggest that this is a justified assumption on random input, our results here show that it doesn't hold universally.

To verify that this is a result of our \emph{assumption} being violated, and not the content of our theorem itself, we experiment with \emph{enforcing} that the feature maps are low-band dominant by inserting low-pass filters after every residual block.
In these models, as expected, we find that we are unable to find low-frequency perturbations which do not significantly increase the feature space distance, unlike for unconstrained or high frequency perturbations.
Adding these filters on the feature maps, however, has a large cost in terms of accuracy.
We leave whether we can enforce the conditions of our theorem in a less restrictive way to future work.
\section{Conclusion and Limitations}
\label{section:conclusion}
We have demonstrated that, under mild assumptions, residual networks will be approximately distance preserving on the low-passed portion of their input.
These conditions and results appear to hold in practice, and on natural datasets to hold even when the assumptions of our mathematical proof are relaxed, for instance, by training with a spectral norm coefficient larger than one.
We think that this is a useful framework to explain the good empirical performance of this scheme in practice.
We have restricted ourselves to the study of \emph{existing} residual architecture and normalisation schemes. We hope that future work can build on our analysis to develop models which enforce the conditions outlined here more explicitly.
We hope that our Algorithm \ref{alg:find_counter_examples} will be a useful tool for the designers of such methods.
We have focused on distance-aware learning in order to facilitate better uncertainty quantification in deep networks.
Our main results are theoretical, and so our work has no obvious direct negative societal impact.
Such models have the potential to avoid errors in automated decision making by referring decisions to an expert in the face of uncertainty.

\section*{Acknowledgements}
The authors would like to thank the members of OATML, OxCSML and anonymous reviewers for their feedback during the project. In particular, we would like to thank Milad Alizadeh, Clare Lyle, Andrew Jesson and Jannik Kossen for helpful discussions. LS/JvA are grateful for funding by the EPSRC (grant reference EP/L015897/1 and EP/N509711/1 respectively). JvA is also grateful for funding by Google-DeepMind.

\bibliographystyle{plainnat}
\bibliography{paper}

\newpage

\appendix

\section{Background: ResNets and invertiblity.}

Here, we restate some results from the literature on invertible models which are relevant to our analysis.
In particular, the following two results are relevant;

\begin{thm}[\citet{behrmann2019invertible}]
	\label{thm:invertibleresnet}
	Let $f_\theta = f_{\theta_1} \circ f_{\theta_2} \circ \dots f_{\theta_n}$ be a ResNet, where each $f_{\theta_i}$ is a residual block $f_{\theta_i}(x) = x + g_{\theta_i}(x)$.
	Then a sufficient condition for $f$ to be invertible is for each $g$ to be a contraction, that is $\lip(g_{\theta_i}) < 1 \forall i$.
\end{thm}

and the following result about the Lipschitz constant

\begin{thm}[\citet{behrmann2019invertible}]
	\label{thm:resblocklipschitz}
	Let $f = I + g$ be a residual block, and let $\lip(g) = L < 1$. Let $||x||$ denote the Euclidean norm of the vector $x$. $f$ is bi-Lipschitz, with
	\begin{align}
		\frac{1}{1 + L} || x - y|| \le ||f(x) - f(y)|| \le (1 + L) ||x - y||
	\end{align}
\end{thm}

These results are important because they establish sufficient conditions for a network mapping to be bi-Lipschitz.
This motivated the regularisation scheme used in practice by existing work on distance aware learning, namely enforcing a Lipschitz constant on the residual connection $g$ using spectral normalisation.
However, it is important to note that these results assume that $f$ is a mapping from $\R^n$ to itself.
As we shall explain in the Section \ref{section:no_bilip}, it is straightforward to show that they \emph{cannot} apply, in general, if this is not the case.

These results are also an important structural motivation for our argument - broadly speaking, the argument in section \ref{section:frequency_maths} tries to establish a sense in which the residual connection can be a contraction if we consider it is restriction to the low frequency components of the input.

\section{Background: The-Nyquist Shannon Theorem and Aliasing}
\label{section:nyquistshannon}

In this section, we give a more detailed discussion of the Nyquist-Shannon theorem and aliasing, and also treat some of the subtleties of this result in more detail than was possible in the main body.

The classical application of the Nyquist-Shannon theorem is the digitalisation of a continuous signal; if a signal sampled at a frequency $u_s$ contains no frequency content above the Nyquist frequency $u_s / 2$ then it can be reconstructed exactly.
However, the result applies equally to re-sampling at a lower frequency.

It is worth mentioning that the Nyquist-Shannon theorem is a \emph{sufficient} condition for lossless reproduction, not a necessary one; with some additional assumptions it is possible to resolve an `undersampled' signal.
The most notable application of this is in compressed sensing, where a signal can sometimes be recovered exactly if it is known that it is sparse in the spectral domain.

\begin{figure}
	\centering
	\includegraphics[width=\textwidth]{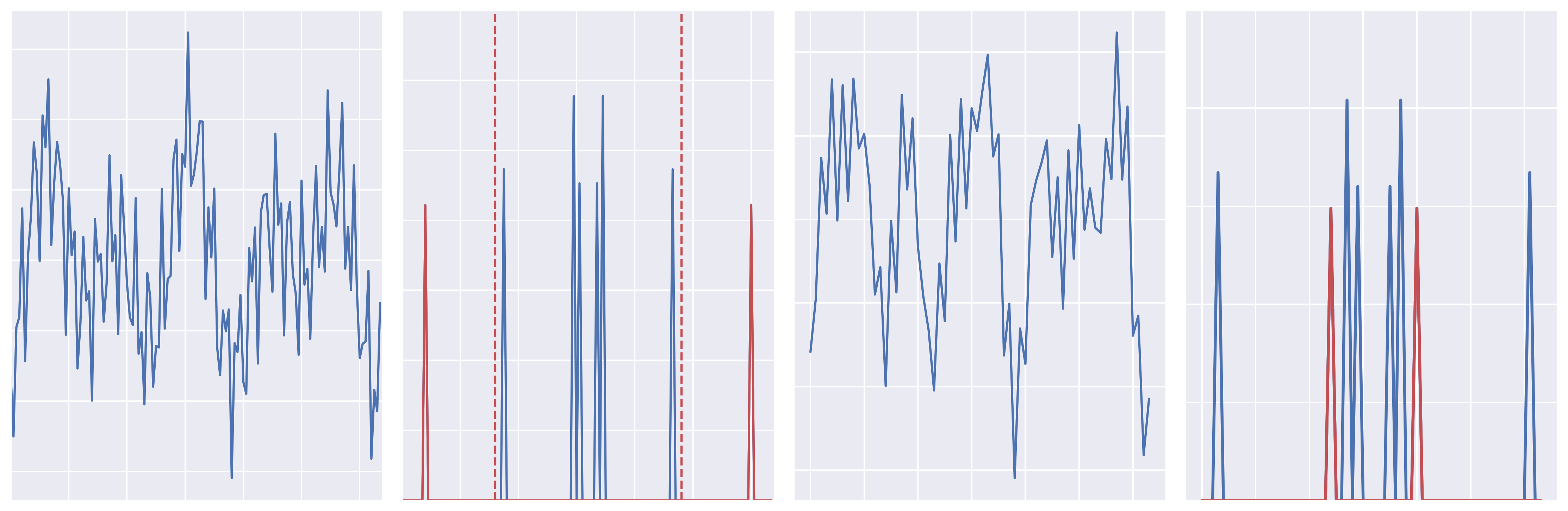}
	\includegraphics[width=\textwidth]{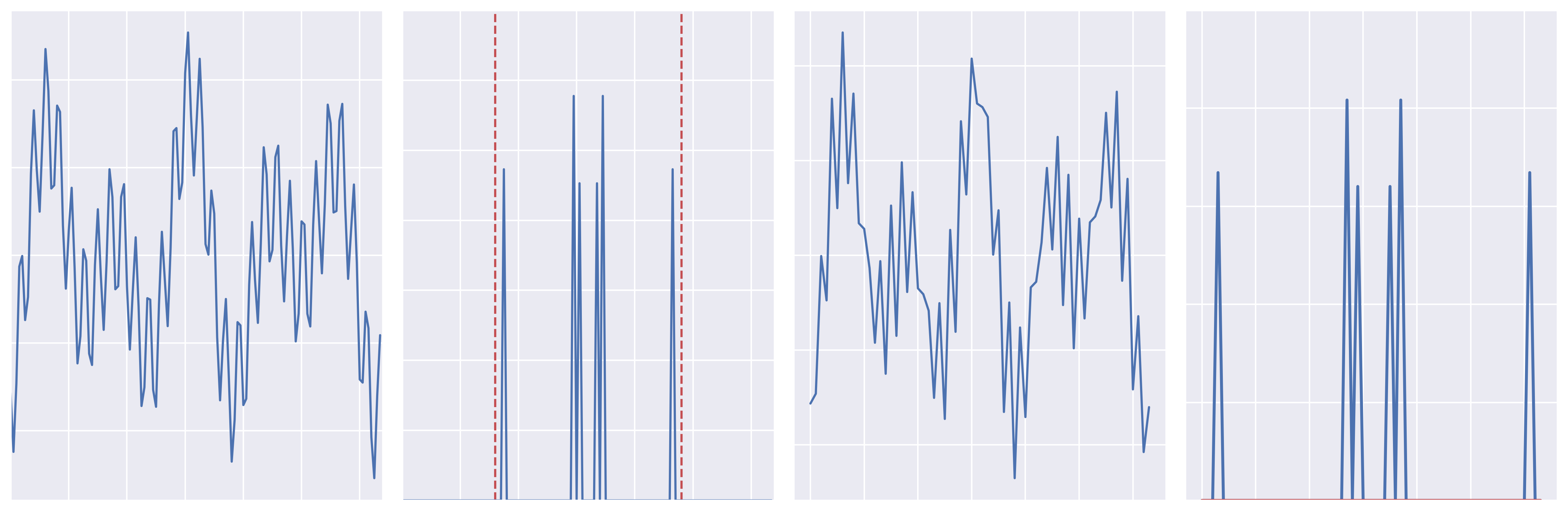}
	\caption{
		Figure illustrating aliasing in a 1d signal after decimation by a factor of 2.
		From left to right; a) the original signal. b) the Fourier spectrum of the original signal, with the Nyquist frequency after downsampling marked as a red dotted line and content above the Nyquist marked in red. Frequency content above this limit will be aliased after downsampling.
		c). the decimated signal. d) the Fourier spectrum of the decimated signal. The peak due to the aliasing of the high frequency component of the original signal is marked in red.
		The second row shows the same, but with a low-pass filter applied to the original signal to remove content above the Nyquist limit before downsampling. The frequency content above the Nyquist is simply lost, but the sub-Nyquist content passes through undistorted.}
	\label{fig:alias_fold}
\end{figure}

In the main body, we refer to aliasing and to low-passing the image to prevent this phenomenon. An illustration of the effect on aliasing on the frequency spectrum of a (1d) signal is shown in Figure \ref{fig:alias_fold}.
An anti-aliasing filter is one like that in the second row of this figure; we remove all frequency above the Nyquist rate to reduce any aliasing effects.
Of course, this still causes feature collapse (indistinguishable signals), but the resultant feature collapse is far easier to reason about, as well as being better behaved; uncontrolled aliasing introduces non-local artefacts into a signal, which in fact hurts the translation invariance of convolutional classifiers \citep{zhang2019making}.
A simple dimensionality argument shows that the volume of potential signals made indistinguishable by these two operations is the same. Say we have a linear mapping $D$ from a signal in $\R^n$ to one in $\R^{n/2}$.
The rank-nullity theorem implies that each signal $y \in \R^{n/2}$ has a space of dimension $\R^{n/2}$ of solutions $x$ to $y = D x$.
This applies whether $D$ is naive decimation or anti-alias filtering + naive decimation, as both of these are linear operations.

In our paper, we often assume `ideal' low pass filtering, that is, a filter which lets all content below a given frequency pass through unchanged, while setting all high frequency content to zero.
This is possible for discrete signals (by implementing this as a direct binary mask in the Fourier domain), but is rarely done.
For continuous signals this ideal frequency response cannot be realised.
In practical image processing, this brick wall filter often causes undesirable `ringing' artefacts in the final image.
In addition, implementing it for discrete signals requires a Fourier transform, rather than using the direct convolution; although the Fourier transform implementation actually has lower asymptotic complexity than a direct convolution, in practice the direct operation is highly optimised on modern GPUs and can often be faster.
\citet{zhang2019making}'s method BlurPool uses a direct convolution kernel to implement this anti-aliasing filter step before all downsampling operations.

\begin{figure}
	\centering
	\includegraphics[width=0.8\textwidth]{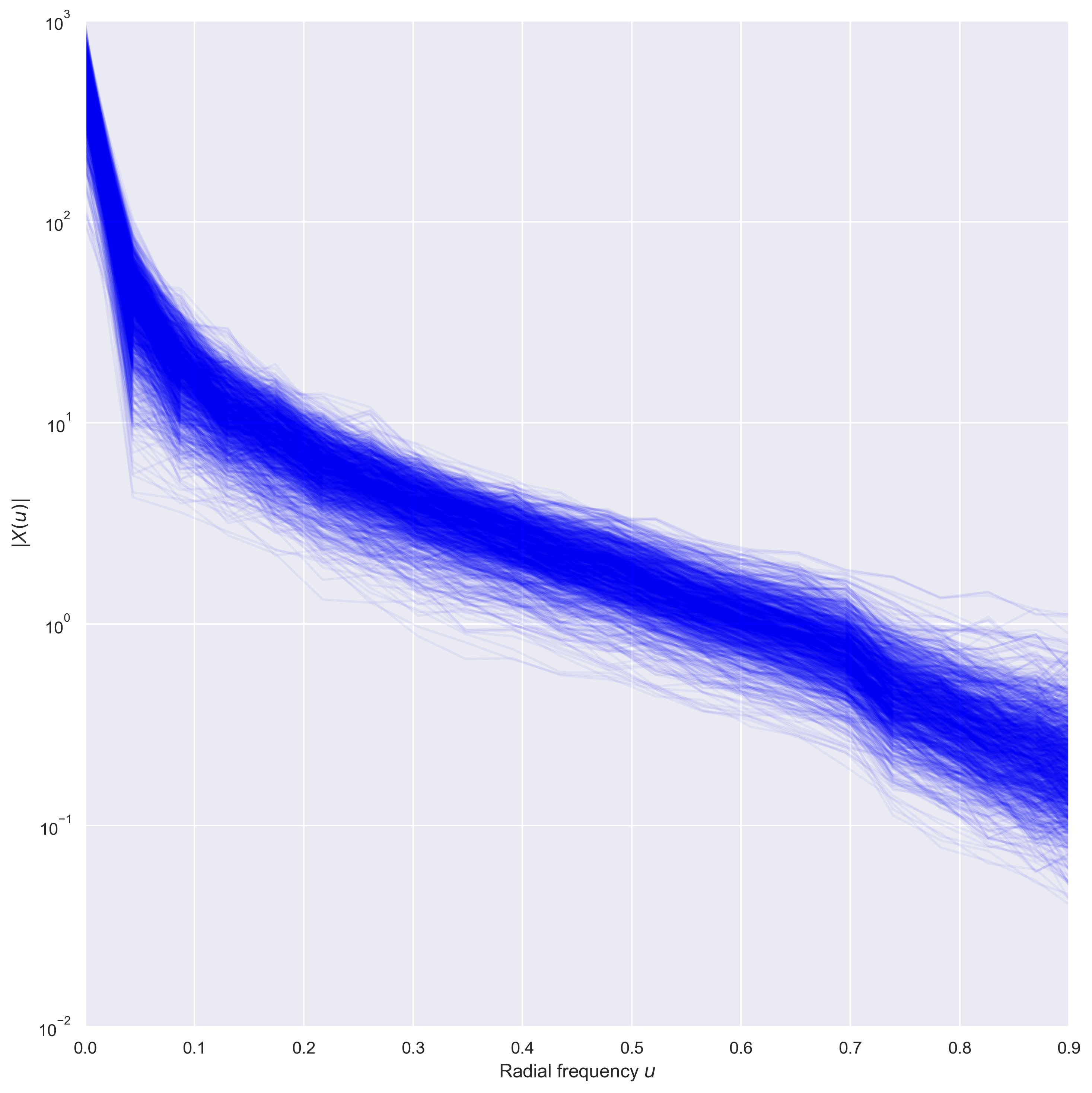}
	\caption{
		The power in each frequency bin for a random sample of images from the CIFAR10 test set.
		In a 2d Fourier transform, the frequency is a 2d vector $f = (f_x, f_y)$. Here, we treat all frequencies with the same radial frequency $\sqrt{f_x^2 + f_y^2}$ as equivalent (i.e we only care about the frequency, not the direction).
		This is an empirical demonstration of our assertion that typical images have most of their power concentrated in the low frequencies.
	}
	\label{fig:freq_spectrum_cifar}
\end{figure}

In the text, we claim that the power of typical images is concentrated in the low frequencies.
In Figure \ref{fig:freq_spectrum_cifar} we show the spectra of a random sample of images, showing that indeed most have their power concentrated in the low frequencies, though we do not find that they are universally low frequency dominant in the sense of Lemma \ref{lem:conv_concentration}, as shown in Tables \ref{table:theorem_checker} and \ref{tab:check_theorems_cifar10}.
However, this does motivate why this assumption of `low frequency dominance' is not unreasonable for natural images.
In addition, this does suggest that even though the low frequency interval may not \emph{strictly} be dominant in the sense we use to prove Lemma \ref{lem:conv_concentration}, as we point out in section \ref{section:frequency_maths} this is likely to be a weak sufficient condition, and if images power is broadly concentrated in the low frequencies it seems reasonable to expect that the ReLU (since it is a convolution in the frequency domain) will still reduce the distances in the low frequency band, which could be an explanation for why we see low-frequency contraction holding far more often than our sufficient condition.

\section{Proofs omitted from the main text}
\label{appendix:proofs}
\subsection{Proofs for section \ref{section:no_bilip}}
\begin{proof}[Proof of theorem \ref{thm:bilip}]

	First, we prove the following lemma.

	\begin{lem}
		Let $f : X \mapsto Y$ be a bi-Lipschitz and surjective function, so $\frac{1}{L}||x_1 - x_2||_X \le ||f(x_1) - f(x_2)||_Y \le L || x_1 - x_2||_X$, and $Y$ is the image of $X$ under $f$. Then $f$ is a homeomorphism between $X$ and $Y$, and so $X$ and $Y$ are homeomorphic.
		\label{lem:bilip_homeomorphism}
	\end{lem}
	\begin{proof}
		Recall that a function $f$ is a homeomorphism if $f$ is bijective, $f$ is continuous, and $f^{-1}$ is also continuous.
		We will address these in turn.
		To see that $f$ is bijective, note that $f$ is injective iff. $\forall x_1, x_2 \in X$, we have $x_1 \ne x_2 \implies f(x_1) \ne f(x_2)$. But this follows directly from the lower Lipschitz property of $f$, since if $x_1 \ne x_2$, then $||x_1 - x_2||_X > 0$, so $|| f(x_1) - f(x_2)||_Y > 0$, from which it follows that $f(x_1) \ne f(x_2)$. Since $f$ is injective (one-to-one) and surjective, it is a bijection.
		Since any function which is Lipschitz continuous is also continuous, the fact that $f$ is continuous is given.
		Since $f$ is bijective, the inverse function $f^{-1}$ exists, and we need to show that it is continuous.
		We have, from the bi-Lipschitzness of $f$, that $\frac{1}{L}||f^{-1}(f(x_1)) - f^{-1}(f(x_2))||_X \le ||f(x_1) - f(x_2)||_Y$, which implies that the inverse function is also Lipschitz, and hence also continuous.
		So $f$ is a homeomorphism, and we are done.
	\end{proof}

	For completeness, we also establish the following lemma, which is a well known corollary of the invariance of domain theorem;
	\begin{lem}
		If $n > m$, and $U$ is a non-empty open subset of $\R^n$, there is no injective, continuous mapping between $U$ and $\R^m$. In particular, therefore, $\R^n$ and $\R^m$ are not homeomorphic.
		\label{lem:invariance_of_dimension}
	\end{lem}
	\begin{proof}
		Recall the invariance of domain;

		\begin{thm}[\citet{brouwer1911beweis}]
			Let $U$ be an open subset of $\R^n$, and let $f : U \to \R^n$ be an injective continuous map.
			Then $V = f(U)$ is also open in $\R^n$.
			\label{thm:invariance_of_domain}
		\end{thm}
		which we will take as given.
		Suppose an injective, continuous function $f : U \to \R^m$ existed. It follows, then, that we could also consider an extension of $f$ as mapping from $\R^n \to \R^n$, by composing with a function of the form $[x_1, x_2, x_3, \dots x_m] \in \R^m \to [x_1,x_2, \dots x_m, 0, 0, \dots 0] \in \R^n$, for example. $f$ would then map from $U$ to the hyperplane $\R^m \subset \R^n$.
		But a hyperplane is not an open set; for any point $x$ in the hyperplane there is a point $x + \epsilon$ that is \emph{not} in the hyperplane, where we can make $\epsilon$ arbitrarily small.
		Using Theorem \ref{thm:invariance_of_domain}, however, we know that if $f$ is continuous and injective, then its image must be an open set.
		We therefore conclude by contradiction that $f: U \to \R^m$ cannot be a continuous, injective mapping, and so cannot be a homeomorphism.

	\end{proof}
	We have therefore established that if a bi-Lipschitz (and surjective) function exists between two spaces $X$ and $Y$, then $X$ and $Y$ are homeomorphic (Lemma \ref{lem:bilip_homeomorphism}).
	This provides a proof of our claim by contradiction; if a function $f : \R^n \mapsto \R^m$ existed and was bi-Lipschitz, then we would have shown that $\R^n$ and $\R^m$ were homeomorphic.
	But Lemma \ref{lem:invariance_of_dimension} shows that this is impossible, and therefore no such function can exist.

\end{proof}

\subsection{Proofs for section \ref{section:frequency_maths}}
\begin{proof}[ Proof of Lemma \ref{lem:conv_concentration} ]
	The proof of this is fairly straightforward, given the assumptions.
	Recall that these are that $f, g$ are normalised measures, and there exists a dominant interval $[x, x + L]$ in $f$ such that $C = \int_x^{x + L} f(x) \ge \int_y^{y + L} \forall y$.
	The mass of $f * g$ in the interval is
	\begin{align*}
		\int_x^{x+L} [f * g](x) dx & = \int_x^{x + L} \int_{-\infty}^{\infty} g(t) f(x - t) dt dx \\
		                           & =  \int_{-\infty}^{\infty} g(t) \int_x^{x+L} f(x - t) dx dt  \\
	\end{align*}
	However, by our concentration assumption, $\int_x^{x+L} f(x - t) dx = \int_{x+t}^{x + t + L} f(x') dx' \le C$.
	Therefore,
	\begin{align*}
		\int_x^{x+L} [f * g](x) dx & \le \int_{-\infty}^{\infty} g(t) C dt \\
		                           & \le C                                 \\
	\end{align*}
	due to the normalisation of $g$.
\end{proof}

\begin{proof}[Proof of Lemma \ref{lem:relu_contraction}]
	We can show that this result using the properties of a convolution, Lemma \ref{lem:conv_concentration}, and a construction reminiscent of the proof of the Mean Value Theorem.
	We need to establish some preliminaries.
	To begin, recall that a convolution $K * X$ is a linear operator, and denote simply by $[K *]$ the Toeplitz or block Toeplitz matrix that implements the convolution as a matrix multiplication.
	Recall also that, using this notation, $\nabla_X K * X = [K*]$
	Now, note that, using the fundamental theorem of calculus, we have that, for
	a general differentiable vector valued function $g(x)$ with Jacobian matrix $J(x)$,
	\begin{align*}
		g(x + v) - g(x) & = \int_{0}^{1} J(x + \lambda v) \cdot v d\lambda
	\end{align*}
	for any $x$ and $v$.
	We can use this to write the action of the ReLU in the frequency domain as
	\begin{align}
		\relu(X + V) - \relu(X) = \int_0^1 M(x + \lambda v) * V d\lambda
	\end{align}
	where we have used the fact that $\relu(X) = M(x) * X$, and so the Jacobian of the ReLU is just $[M(x) *]$ (another way to see this is to note that the ReLU is piecewise linear, so is always locally equal to its Jacobian).
	$M(x)$ denotes the Fourier transform of the image mask $ I(x > 0) $, as before.
	We are free to use this in the frequency domain, as the Fourier transform is simply a change of basis.
	Then, taking norms,
	\begin{align*}
		||\relu(X + V) - \relu(X)|| & = || \int_0^1 M(x + \lambda v) * V d\lambda ||    \\
		                            & \le  \int_0^1 || M(x + \lambda v) * V || d\lambda \\
	\end{align*}
	Where the second line follows from noting that the norm is convex and applying Jensen's inequality.

	Now, along the trajectory parameterised by $\lambda$ there must exist a maximal point $\lambda^\star := \arg\max_\lambda ||M(x + \lambda v) * V||$.
	We can then upper bound the integral by substituting this maximal value for the value of $||M(x + \lambda v) * V||$ in the integrand, that is
	\begin{align*}
		||\relu(X + V) - \relu(X)|| & \le  \int_0^1 || M(x + \lambda v) * V || d\lambda                                      \\
		                            & \le  ||M(x + \lambda^\star v) * V|| \int_0^1 d\lambda = ||M(x + \lambda^\star v) * V||
	\end{align*}

	So we can upper bound the distance in output space with the norm of a convolution.
	Now, let $Z = M(x + \lambda^\star v) * V$, and consider splitting both $Z$ into high and low frequency components, so $Z_l = H_u(Z), Z_h = Z - Z_l$.
	Note that these vectors are orthogonal and $Z = Z_l + Z_h$.
	We define $V_h$ and $V_l$ as the high and low components of $V$ in the same way.
	Note that the vector norm is invariant to taking absolute values elementwise, so we can consider the norm of $|Z| = |M(x + \lambda^\star v)| * |V|$, which allows these vectors to be treated as unnormalised measures. The fact that we can take the magnitude inside the convolution follows from the definition of convolution.

	We assumed that $V$ is low-band dominant, as in the conditions of Lemma \ref{lem:conv_concentration}, so $V_l$ is a dominant interval.
	It follows then that, applying this lemma,
	\begin{align*}
		\frac{||Z_l||^2}{||Z||} & \le \frac{||V_l||}{||V||}       \\
		||Z_l||^2               & \le ||V_l||\frac{||Z|| }{||V||}
	\end{align*}

	Now, recall that, by construction, $Z = M(x + \lambda^\star v)  * V$ is the inner product of the Jacobian of the ReLU with a vector $V$, that is, the matrix $[M(x + \lambda^\star v) *]$ is the Jacobian of the ReLU, evaluated at some point $x$.
	But the ReLU has a Lipschitz constant of 1, so the operator norm of the Jacobian must be less than or equal 1.
	Therefore, $\frac{||Z||}{||V||} \le 1$, and we are done.

\end{proof}

\begin{proof}[Proof of theorem \ref{thm:low_pass_bilip}]
	Despite the fact that this theorem is a more important result, its proof is essentially a corollary of Lemma \ref{lem:relu_contraction}.
	First, note that if this lemma applies, the ReLU reduces the low-frequency distances, so $||H_u(\relu(x))  - H_u(\relu(y)) || \le ||H_u(x) - H_u(y)||$.
	Convolutions, if we apply spectral normalisation so that the we have $\lip(g_i) < 1$ for all convolutions $g_i$ in the branch, must also reduce the low-frequency distances (since they act on all components independently),
	and so we must also have $||H_u(g(x)) - H_u(g(y))|| \le L ||H_u(x) - H_u(y)||$ for some $L < 1$, since $g$ is a composition of functions for which this is true.
	Using an argument similar to that in \citet{behrmann2019invertible}, we then can
	use the residual structure to obtain
	\begin{align*}
		||H_u(f(x)) - H_u(f(y))|| & = ||H_u(x + g(x)) - H_u(y + g(y))||                                                                                  \\
		                          & = || H_u(x) - H_u(y) - (H_u(g(x)) - H_u(g(y))) ||                                                                    \\
		                          & \ge ||H_u(x) - H_u(y)|| - ||H_u(g(x)) - H_u(g(y))|| &  & \textrm{Reverse triangle inequality}                        \\
		                          & \ge (1 - L)||H_u(x) - H_u(y)||                      &  & \textrm{Assuming lemma \ref{lem:relu_contraction} applies }
	\end{align*}
\end{proof}

\section{Related Work}

On aliasing and downsampling in neural networks, the work by \citet{zhang2019making} is the most relevant.
In this paper, we use their insights and suggested method, BlurPool.
In comparison, \citet{zhang2019making} do not discuss or analyse the impact of aliasing on uncertainty.

Spectral normalisation is used in a variety of situations, \citet{rosca2020case} show it has a beneficial effect for generalisation, robustness and GAN critics among others.
In this paper, we focus on using spectral normalisation in combination with residual connections for a bi-Lipschitz constraint.
This particular setup was introduced in \citet{liu2020simple}, and extended upon in \citet{mukhoti2021deterministic,van2021improving}, who show that this technique can obtain state of the art uncertainty estimation for classification and regression.
The core idea of all three methods is to compute a distance in feature space, this can be done through for example a Gaussian process or a kernel density.
A low likelihood under the density model indicates high uncertainty, and a Gaussian process with stationary kernel naturally increases its uncertainty for points that are far away of the inducing points.

The first practical, fully bi-Lipschitz (i.e. without downsampling) models, were introduced in \citet{behrmann2019invertible} who show that a standard ResNet architecture can be invertible if one avoids downsampling and enforces an upper Lipschitz constraint on the main branch of every residual block.
Inversion is done using fixed point iteration.
\citet{lu2021implicit} relax the Lipschitz constraint at the expense of more complex inversions, and \citet{sander2021momentum} introduce a momentum term which simplifies the ResNet architecture changes.
None of these methods considers the situation of downsampling, which makes them impractical to scale beyond low resolution images such as CIFAR-10

\section{Alternative feature collapse algorithm}

Here, we present an alternative feature collapse search algorithm.
This algorithm takes steps in the null space of the Jacobian, rather than forcing the function to be orthonormal to the gradient of $||y - y_0||$ as in Algorithm \ref{alg:find_counter_examples}.
This is straightforward; projecting to keep $||y - y_0||$ constant is only a heuristic, and more restrictive than it needs to be.
However, this algorithm requires to explicitly compute the right singular vectors of the Jacobian of the feature mapping.
For large feature spaces this is prohibitively expensive and not practical.
This algorithm is only provided for completeness.

\begin{algorithm}
	\caption{Given image $x_0$, find an image $x$ such that $f(x_0) \simeq f(x)$ and $||x_0 - x|| > 0$.}
	\label{alg:find_counter_examples_2}
	\begin{algorithmic}[0]
		\State \textbf{Input:} Input image $x_0$, feature mapping $f$, step size $\eta$, number of steps $n$, distribution over unit vectors $D$
		\State \textbf{Output:} Example image $x$
		\State $y_0 \gets f(x_0)$
		\State $r \gets D()$
		\State $x \gets x_0 + \eta v$
		\For{ $i \in 0..n-1$ }
		\State $r \gets D()$ \Comment{Sample random step direction}
		\State $J \gets \nabla_x f(x) $  \Comment{find the Jacobian of the feature mapping}
		\State $\_, \_, V \gets \textrm{svd}(J)$ \Comment{Compute the (compact) right singular vectors of the Jacobian (i.e those with non-zero singular values)}
		\State $u \gets r - \sum_i V[:, i]  (\langle V[:, i], r \rangle) $ \Comment{Project the step direction into the null space of the Jacobian}
		\State $x \gets x + \eta u$ \Comment{Take a step along the level set}
		\EndFor \\
		\Return $x$
	\end{algorithmic}
\end{algorithm}

\section{Additional Experimental Results}
\label{appendix:more_exps}
\FloatBarrier

\subsection{Additional verification results}

In the main paper, we provide the results of checking the conditions of Lemma \ref{lem:conv_concentration} and Theorem \ref{thm:low_pass_bilip} on Fashion MNIST.
Here, in Table \ref{table:theorem_checker} and Table \ref{tab:check_theorems_cifar10}, we provide analogous results for CIFAR10 and MNIST in addition.
These were omitted from the main paper due to space considerations, but the results are consistent across the three datasets tested.

\begin{table}[t]
	\centering
	\caption{Results of empirically checking the conditions of Lemma \ref{lem:conv_concentration} (low-frequency
		domination) and Theorem \ref{thm:low_pass_bilip} (low frequency contraction).
		Numbers are the proportion of the dataset for which we found the condition to hold exactly; so 1 means no violations were found, 0.5 means it was true for half the inputs tested, etc.
		We report means and standard error over 25 seeds for these datasets.
		We use a WideResNet with depth 10 and widen factor of 1 for these datasets.
	}
	\label{table:theorem_checker}
	\resizebox{0.9\columnwidth}{!}{%
		\begin{tabular}{@{}lllccccc@{}}
			\toprule
			                                                     & Dataset                 & $\lip(g)$        & $x$              & $f_1(x)$         & $f_2(x)$         & $f_3(x)$         & $f_4(x)$         \\ \midrule
			\multirow{14}{*}{Lemma \ref{lem:conv_concentration}} & \multirow{7}{*}{MNIST}  & 0.9              & $0.80  \pm 0.00$ & $0.18  \pm 0.04$ & $0.31  \pm 0.03$ & $0.63  \pm 0.02$ & $0.44  \pm 0.02$ \\
			                                                     &                         & 0.95             & $0.80  \pm 0.00$ & $0.18  \pm 0.04$ & $0.30  \pm 0.05$ & $0.62  \pm 0.04$ & $0.45  \pm 0.02$ \\
			                                                     &                         & 0.99             & $0.80  \pm 0.00$ & $0.20  \pm 0.03$ & $0.31  \pm 0.03$ & $0.63  \pm 0.02$ & $0.45  \pm 0.01$ \\
			                                                     &                         & 3.0              & $0.80  \pm 0.00$ & $0.23  \pm 0.02$ & $0.43  \pm 0.04$ & $0.71  \pm 0.01$ & $0.47  \pm 0.02$ \\
			                                                     &                         & 6.0              & $0.80  \pm 0.00$ & $0.22  \pm 0.03$ & $0.42  \pm 0.03$ & $0.71  \pm 0.01$ & $0.48  \pm 0.02$ \\
			                                                     &                         & 9.0              & $0.80  \pm 0.00$ & $0.22  \pm 0.04$ & $0.41  \pm 0.03$ & $0.71  \pm 0.01$ & $0.48  \pm 0.01$ \\
			                                                     &                         & no               & $0.80  \pm 0.00$ & $0.24  \pm 0.02$ & $0.39  \pm 0.05$ & $0.70  \pm 0.01$ & $0.46  \pm 0.02$ \\ \cmidrule(l){2-8}
			                                                     & \multirow{7}{*}{FMNIST}
			                                                     & 0.9                     & $0.86  \pm 0.00$ & $0.18  \pm 0.04$ & $0.24  \pm 0.04$ & $0.48  \pm 0.04$ & $0.47  \pm 0.03$                    \\
			                                                     &                         & 0.95             & $0.86  \pm 0.00$ & $0.18  \pm 0.04$ & $0.25  \pm 0.02$ & $0.48  \pm 0.03$ & $0.46  \pm 0.01$ \\
			                                                     &                         & 0.99             & $0.86  \pm 0.00$ & $0.18  \pm 0.05$ & $0.25  \pm 0.04$ & $0.51  \pm 0.02$ & $0.48  \pm 0.01$ \\
			                                                     &                         & 3.0              & $0.86  \pm 0.00$ & $0.24  \pm 0.03$ & $0.28  \pm 0.04$ & $0.53  \pm 0.04$ & $0.47  \pm 0.02$ \\
			                                                     &                         & 6.0              & $0.86  \pm 0.00$ & $0.24  \pm 0.03$ & $0.28  \pm 0.02$ & $0.53  \pm 0.02$ & $0.47  \pm 0.02$ \\
			                                                     &                         & 9.0              & $0.86  \pm 0.00$ & $0.25  \pm 0.01$ & $0.28  \pm 0.04$ & $0.54  \pm 0.04$ & $0.47  \pm 0.01$ \\
			                                                     &                         & no               & $0.86  \pm 0.00$ & $0.24  \pm 0.03$ & $0.25  \pm 0.03$ & $0.53  \pm 0.03$ & $0.47  \pm 0.01$ \\
			\midrule
			\multirow{14}{*}{Theorem \ref{thm:low_pass_bilip}}   & \multirow{7}{*}{MNIST}  & 0.9              & n/a              & $1.00 \pm 0.00$  & $1.00 \pm 0.00$  & $1.00 \pm 0.00$  & $1.00 \pm 0.00$  \\
			                                                     &                         & 0.95             & n/a              & $1.00 \pm 0.00$  & $1.00 \pm 0.00$  & $1.00 \pm 0.00$  & $1.00 \pm 0.00$  \\
			                                                     &                         & 0.99             & n/a              & $1.00 \pm 0.00$  & $1.00 \pm 0.00$  & $1.00 \pm 0.00$  & $1.00 \pm 0.00$  \\
			                                                     &                         & 3.0              & n/a              & $0.69 \pm 0.31$  & $0.96 \pm 0.11$  & $1.00 \pm 0.00$  & $1.00 \pm 0.00$  \\
			                                                     &                         & 6.0              & n/a              & $0.55 \pm 0.26$  & $0.98 \pm 0.05$  & $1.00 \pm 0.00$  & $1.00 \pm 0.00$  \\
			                                                     &                         & 9.0              & n/a              & $0.61 \pm 0.28$  & $0.99 \pm 0.03$  & $1.00 \pm 0.00$  & $1.00 \pm 0.00$  \\
			                                                     &                         & no               & n/a              & $0.60 \pm 0.30$  & $0.99 \pm 0.02$  & $1.00 \pm 0.00$  & $1.00 \pm 0.00$  \\ \cmidrule(l){2-8}
			                                                     & \multirow{7}{*}{FMNIST} & 0.9              & n/a              & $1.00 \pm 0.00$  & $1.00 \pm 0.00$  & $1.00 \pm 0.00$  & $1.00 \pm 0.00$  \\
			                                                     &                         & 0.95             & n/a              & $1.00 \pm 0.00$  & $1.00 \pm 0.00$  & $1.00 \pm 0.00$  & $1.00 \pm 0.00$  \\
			                                                     &                         & 0.99             & n/a              & $1.00 \pm 0.00$  & $1.00 \pm 0.00$  & $1.00 \pm 0.00$  & $1.00 \pm 0.00$  \\
			                                                     &                         & 3.0              & n/a              & $0.98 \pm 0.04$  & $1.00 \pm 0.00$  & $1.00 \pm 0.00$  & $1.00 \pm 0.00$  \\
			                                                     &                         & 6.0              & n/a              & $0.95 \pm 0.11$  & $0.92 \pm 0.04$  & $1.00 \pm 0.00$  & $1.00 \pm 0.00$  \\
			                                                     &                         & 9.0              & n/a              & $1.00 \pm 0.01$  & $0.85 \pm 0.12$  & $1.00 \pm 0.00$  & $1.00 \pm 0.00$  \\
			                                                     &                         & no               & n/a              & $1.00 \pm 0.00$  & $0.86 \pm 0.07$  & $1.00 \pm 0.00$  & $1.00 \pm 0.00$  \\ \bottomrule
		\end{tabular}
	}
\end{table}

\begin{table}
	\centering
	\caption{Verifying theories on CIFAR10 - see the caption of Table \ref{table:theorem_checker} for a more detailed description.
		We report means and standard errors over 5 seeds for this dataset.
		We use a WideResNet with a depth of 28 and widen factor of 10.
	}
	\label{tab:check_theorems_cifar10}
	\resizebox{1.0\columnwidth}{!}{%
		\begin{tabular}{llccccccc}
			\toprule
			\toprule
			 & $\lip(g)$ & $x$              & $f_1(x)$         & $f_2(x)$         & $f_3(x)$         & $f_4(x)$         & $f_5(x)$         & $f_6(x)$         \\
			\cmidrule(l){2-9}
			\multirow{14}{*}{Lemma 1}
			 & 0.9       & $0.66  \pm 0.00$ & $0.53  \pm 0.11$ & $0.56  \pm 0.20$ & $0.41  \pm 0.13$ & $0.41  \pm 0.13$ & $0.41  \pm 0.12$ & $0.67  \pm 0.01$ \\
			 & 0.95      & $0.66  \pm 0.00$ & $0.53  \pm 0.03$ & $0.41  \pm 0.19$ & $0.35  \pm 0.07$ & $0.37  \pm 0.06$ & $0.39  \pm 0.04$ & $0.67  \pm 0.06$ \\
			 & 0.99      & $0.66  \pm 0.00$ & $0.52  \pm 0.03$ & $0.35  \pm 0.17$ & $0.30  \pm 0.05$ & $0.30  \pm 0.04$ & $0.33  \pm 0.04$ & $0.65  \pm 0.01$ \\
			 & 3.0       & $0.66  \pm 0.00$ & $0.54  \pm 0.03$ & $0.26  \pm 0.02$ & $0.19  \pm 0.03$ & $0.20  \pm 0.02$ & $0.26  \pm 0.02$ & $0.64  \pm 0.01$ \\
			 & 6.0       & $0.66  \pm 0.00$ & $0.94  \pm 0.00$ & $0.36  \pm 0.03$ & $0.28  \pm 0.09$ & $0.21  \pm 0.06$ & $0.25  \pm 0.03$ & $0.62  \pm 0.02$ \\
			 & 9.0       & $0.66  \pm 0.00$ & $0.88  \pm 0.13$ & $0.43  \pm 0.04$ & $0.27  \pm 0.05$ & $0.20  \pm 0.04$ & $0.18  \pm 0.02$ & $0.60  \pm 0.01$ \\
			 & no        & $0.66  \pm 0.00$ & $0.92  \pm 0.03$ & $0.41  \pm 0.04$ & $0.24  \pm 0.07$ & $0.17  \pm 0.02$ & $0.18  \pm 0.02$ & $0.60  \pm 0.01$ \\

			\cmidrule(l){2-9}
			 & $\lip(g)$ & $f_7(x)$         & $f_8(x)$         & $f_9(x)$         & $f_{10}(x)$      & $f_{11}(x)$      & $f_{12}(x)$      & $f_{13}(x)$      \\
			\cmidrule(l){2-9}
			 & 0.9       & $0.69  \pm 0.03$ & $0.71  \pm 0.03$ & $0.78  \pm 0.04$ & $0.85  \pm 0.08$ & $0.77  \pm 0.01$ & $0.78  \pm 0.03$ & $0.94  \pm 0.01$ \\
			 & 0.95      & $0.67  \pm 0.03$ & $0.69  \pm 0.02$ & $0.77  \pm 0.05$ & $0.81  \pm 0.09$ & $0.76  \pm 0.02$ & $0.76  \pm 0.03$ & $0.95  \pm 0.02$ \\
			 & 0.99      & $0.67  \pm 0.01$ & $0.69  \pm 0.01$ & $0.75  \pm 0.04$ & $0.81  \pm 0.08$ & $0.77  \pm 0.01$ & $0.78  \pm 0.02$ & $0.93  \pm 0.02$ \\
			 & 3.0       & $0.63  \pm 0.02$ & $0.66  \pm 0.02$ & $0.72  \pm 0.01$ & $0.76  \pm 0.02$ & $0.84  \pm 0.02$ & $0.81  \pm 0.04$ & $0.83  \pm 0.03$ \\
			 & 6.0       & $0.63  \pm 0.01$ & $0.65  \pm 0.01$ & $0.73  \pm 0.02$ & $0.85  \pm 0.07$ & $0.77  \pm 0.01$ & $0.76  \pm 0.05$ & $0.82  \pm 0.04$ \\
			 & 9.0       & $0.63  \pm 0.01$ & $0.66  \pm 0.02$ & $0.69  \pm 0.01$ & $0.69  \pm 0.06$ & $0.76  \pm 0.01$ & $0.71  \pm 0.05$ & $0.85  \pm 0.02$ \\
			 & no        & $0.63  \pm 0.01$ & $0.68  \pm 0.02$ & $0.70  \pm 0.01$ & $0.75  \pm 0.10$ & $0.74  \pm 0.03$ & $0.81  \pm 0.01$ & $0.74  \pm 0.02$ \\
			\midrule
			\midrule
			 & $\lip(g)$ & $x$              & $f_1(x)$         & $f_2(x)$         & $f_3(x)$         & $f_4(x)$         & $f_5(x)$         & $f_6(x)$         \\
			\cmidrule(l){2-9}
			\multirow{14}{*}{Theorem 1}
			 & 0.9       & $\backslash$     & $1.00  \pm 0.00$ & $1.00  \pm 0.00$ & $1.00  \pm 0.00$ & $1.00  \pm 0.00$ & $1.00  \pm 0.00$ & $1.00  \pm 0.00$ \\
			 & 0.95      & $\backslash$     & $1.00  \pm 0.00$ & $1.00  \pm 0.00$ & $1.00  \pm 0.00$ & $1.00  \pm 0.00$ & $1.00  \pm 0.00$ & $1.00  \pm 0.00$ \\
			 & 0.99      & $\backslash$     & $1.00  \pm 0.00$ & $1.00  \pm 0.00$ & $1.00  \pm 0.00$ & $1.00  \pm 0.00$ & $1.00  \pm 0.00$ & $1.00  \pm 0.00$ \\
			 & 3.0       & $\backslash$     & $1.00  \pm 0.00$ & $1.00  \pm 0.00$ & $1.00  \pm 0.00$ & $1.00  \pm 0.00$ & $1.00  \pm 0.00$ & $1.00  \pm 0.00$ \\
			 & 6.0       & $\backslash$     & $1.00  \pm 0.00$ & $1.00  \pm 0.00$ & $1.00  \pm 0.00$ & $1.00  \pm 0.00$ & $1.00  \pm 0.00$ & $1.00  \pm 0.00$ \\
			 & 9.0       & $\backslash$     & $1.00  \pm 0.00$ & $1.00  \pm 0.00$ & $1.00  \pm 0.00$ & $1.00  \pm 0.00$ & $1.00  \pm 0.00$ & $1.00  \pm 0.00$ \\
			 & no        & $\backslash$     & $1.00  \pm 0.00$ & $1.00  \pm 0.00$ & $1.00  \pm 0.00$ & $1.00  \pm 0.00$ & $1.00  \pm 0.00$ & $1.00  \pm 0.00$ \\
			\cmidrule(l){2-9}
			 & $\lip(g)$ & $f_7(x)$         & $f_8(x)$         & $f_9(x)$         & $f_{10}(x)$      & $f_{11}(x)$      & $f_{12}(x)$      & $f_{13}(x)$      \\
			\cmidrule(l){2-9}
			 & 0.9       & $1.00  \pm 0.00$ & $1.00  \pm 0.00$ & $1.00  \pm 0.00$ & $1.00  \pm 0.00$ & $1.00  \pm 0.00$ & $1.00  \pm 0.00$ & $1.00  \pm 0.00$ \\
			 & 0.95      & $1.00  \pm 0.00$ & $1.00  \pm 0.00$ & $1.00  \pm 0.00$ & $1.00  \pm 0.00$ & $1.00  \pm 0.00$ & $1.00  \pm 0.00$ & $1.00  \pm 0.00$ \\
			 & 0.99      & $1.00  \pm 0.00$ & $1.00  \pm 0.00$ & $1.00  \pm 0.00$ & $1.00  \pm 0.00$ & $1.00  \pm 0.00$ & $1.00  \pm 0.00$ & $1.00  \pm 0.00$ \\
			 & 3.0       & $1.00  \pm 0.00$ & $1.00  \pm 0.00$ & $1.00  \pm 0.00$ & $1.00  \pm 0.00$ & $1.00  \pm 0.00$ & $1.00  \pm 0.00$ & $1.00  \pm 0.00$ \\
			 & 6.0       & $1.00  \pm 0.00$ & $1.00  \pm 0.00$ & $1.00  \pm 0.00$ & $1.00  \pm 0.00$ & $1.00  \pm 0.00$ & $1.00  \pm 0.00$ & $1.00  \pm 0.00$ \\
			 & 9.0       & $1.00  \pm 0.00$ & $1.00  \pm 0.00$ & $1.00  \pm 0.00$ & $1.00  \pm 0.00$ & $1.00  \pm 0.00$ & $1.00  \pm 0.00$ & $1.00  \pm 0.00$ \\
			 & no        & $1.00  \pm 0.00$ & $1.00  \pm 0.00$ & $1.00  \pm 0.00$ & $1.00  \pm 0.00$ & $1.00  \pm 0.00$ & $1.00  \pm 0.00$ & $1.00  \pm 0.00$ \\
			\bottomrule
			\bottomrule
		\end{tabular}
	}

\end{table}

\subsection{Theorem check during training}

In Figure \ref{fig:theorem_check_training_full}, we show the same information as Figure \ref{fig:theorem_check_training_small} in the main text, but for all the residual blocks of the network.
This shows that the blocks in that figure are representative of the model as a whole, but we were unable to fit the full figure in the main body.

\begin{figure}
	\centering
	\includegraphics[width=\textwidth]{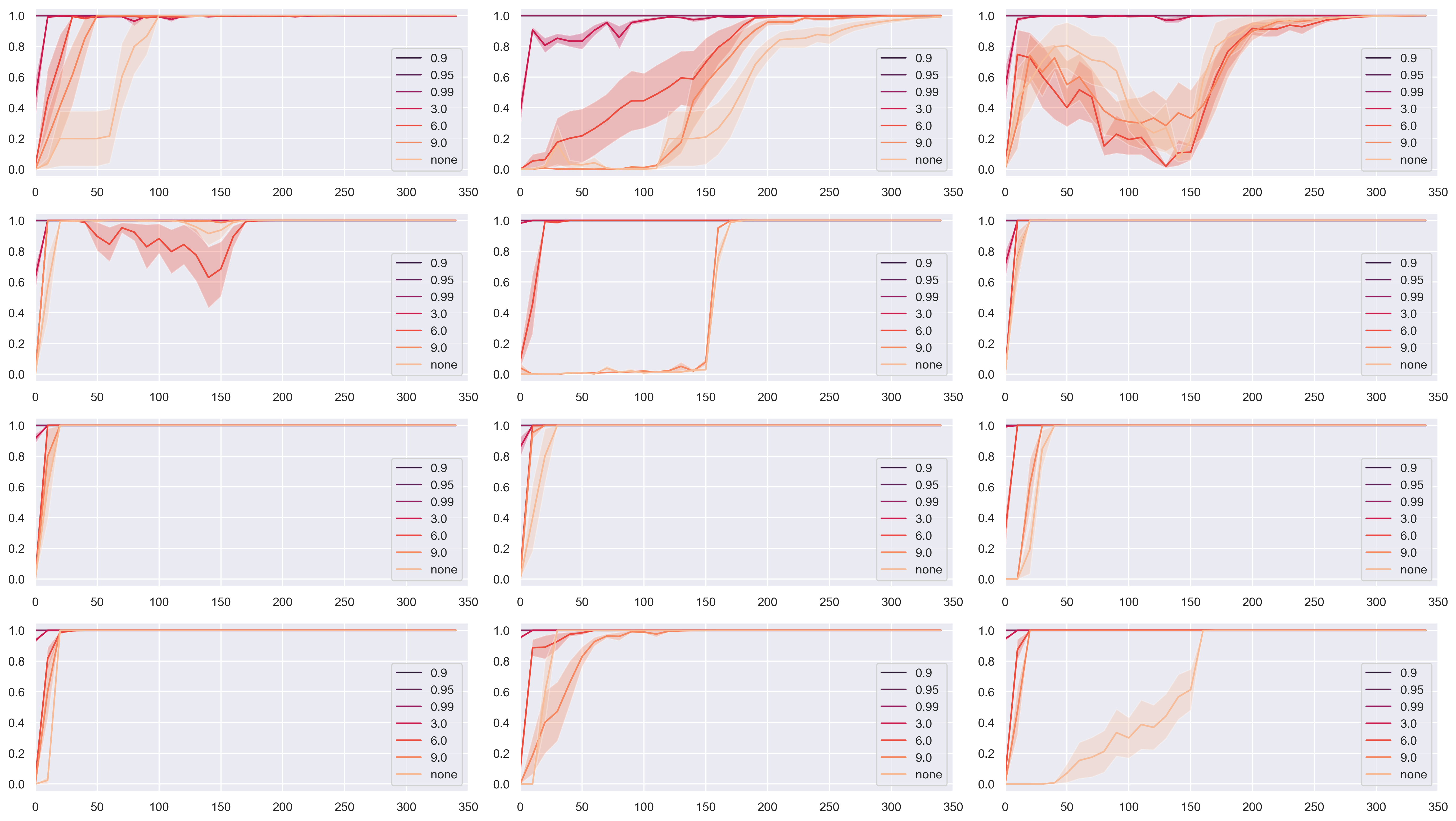}
	\caption{
		The same as Figure \ref{fig:theorem_check_training_small}, but including all blocks of the model.
		As can be seen, the observation of convergence to the result of the theorem mostly being true holds across all residual blocks.
		The blocks are shown in order from left to right and top to bottom.
	}
	\label{fig:theorem_check_training_full}
\end{figure}

\subsection{AUROC and test accuracy}

In Table \ref{tab:test_acc}, we show the accuracy and OOD performance, as measured by AUROC on an out-of distribution dataset, in Table~\ref {tab:auroc}.
These model do not obtain SotA accuracy, but we include them to demonstrate that we achieve competitive accuracy, in order to demonstrate that our training setup is reasonable and our architectural changes do not have a major effect on the performance of our models.

\begin{table}
	\centering
	\caption{OOD detection results (AUROC).}
	\label{tab:auroc}
	\resizebox{1.0\columnwidth}{!}{%
		\begin{tabular}{llccccccc}
			\toprule
			Ind dataset & OoD dataset & $\lip(g)=0.9$     & $\lip(g)=0.95$    & $\lip(g)=0.99$    & $\lip(g)=3.0$     & $\lip(g)=6.0$     & $\lip(g)=9.0$     & No spec norm      \\
			\midrule
			MNIST       & FMNIST      & $99.98  \pm 0.01$ & $99.98  \pm 0.01$ & $99.98  \pm 0.01$ & $99.97  \pm 0.01$ & $99.97  \pm 0.01$ & $99.97  \pm 0.01$ & $99.97  \pm 0.01$ \\
			FMNIST      & MNIST       & $98.75  \pm 0.24$ & $98.50  \pm 0.47$ & $98.67  \pm 0.35$ & $97.89  \pm 1.21$ & $98.25  \pm 0.52$ & $98.10  \pm 0.96$ & $98.28  \pm 0.46$ \\
			CIFAR10     & SVHN        & $93.47  \pm 2.81$ & $91.66  \pm 1.12$ & $93.27  \pm 2.63$ & $94.07  \pm 0.70$ & $92.97  \pm 2.47$ & $92.31  \pm 2.17$ & $95.64  \pm 1.58$ \\
			CIFAR10     & CIFAR100    & $86.92  \pm 0.84$ & $86.75  \pm 0.64$ & $87.28  \pm 0.64$ & $89.74  \pm 0.55$ & $89.22  \pm 0.52$ & $89.72  \pm 0.29$ & $90.16  \pm 0.56$ \\
			\bottomrule
		\end{tabular}
	}
\end{table}

\begin{table}
	\centering
	\caption{Test accuracy of the pre-trained models.}
	\label{tab:test_acc}
	\begin{tabular}{lccc}
		\toprule
		$\lip(g)$ & MNIST             & FMNIST            & CIFAR10           \\
		\midrule
		0.9       & $99.49  \pm 0.05$ & $90.64  \pm 0.35$ & $94.54  \pm 0.13$ \\
		0.95      & $99.55  \pm 0.08$ & $90.74  \pm 0.30$ & $94.64  \pm 0.20$ \\
		0.99      & $99.54  \pm 0.07$ & $90.97  \pm 0.22$ & $94.85  \pm 0.24$ \\
		3.0       & $99.58  \pm 0.08$ & $92.46  \pm 0.22$ & $95.75  \pm 0.08$ \\
		6.0       & $99.56  \pm 0.10$ & $92.97  \pm 2.47$ & $95.54  \pm 0.13$ \\
		9.0       & $99.55  \pm 0.07$ & $92.21  \pm 0.21$ & $95.60  \pm 0.22$ \\
		no        & $99.54  \pm 0.06$ & $92.16  \pm 0.48$ & $95.49  \pm 0.15$ \\
		\bottomrule
	\end{tabular}
\end{table}

\subsection{Feature space attack}

In Table \ref{table:feature_space_distance}, we provide a more detailed numerical supplement to Figure \ref{fig:adv_trajectory}.
This shows results for a standard model, rather than with low-pass filters inserted between the residual blocks.
Consistent with Figure \ref{fig:adv_trajectory}, we observe that for standard models spectral normalisation tends to increase the distance in feature space distance of this attack, but we are unable to observe the low-frequency/high-frequency delta predicted by our theory without explicitly enforcing the domination assumption, suggesting that the adversarial search is able to find counter-examples to this condition.

\begin{table}
	\centering
	\caption{Feature space movement of the adversarial attack on CIFAR10.}
	\label{table:feature_space_distance}
	\begin{tabular}{lcccccc}
		\toprule
		          & \multicolumn{2}{c}{All freq} & \multicolumn{2}{c}{High freq} & \multicolumn{2}{c}{Low freq}                                                          \\
		$\lip(g)$ & $|x-x_0|$                    & $|y-y_0|$                     & $|x-x_0|$                    & $|y-y_0|$        & $|x-x_0|$        & $|y-y_0|$        \\
		\midrule
		0.9       & $4.47  \pm 0.00$             & $0.89  \pm 0.06$              & $4.47  \pm 0.01$             & $0.90  \pm 0.06$ & $4.45  \pm 0.02$ & $0.63  \pm 0.05$ \\
		0.95      & $4.47  \pm 0.01$             & $0.87  \pm 0.03$              & $4.47  \pm 0.00$             & $0.89  \pm 0.03$ & $4.47  \pm 0.02$ & $0.63  \pm 0.03$ \\
		0.99      & $4.47  \pm 0.00$             & $0.92  \pm 0.06$              & $4.47  \pm 0.01$             & $0.92  \pm 0.06$ & $4.48  \pm 0.01$ & $0.64  \pm 0.03$ \\
		3.0       & $4.47  \pm 0.00$             & $0.75  \pm 0.04$              & $4.47  \pm 0.00$             & $0.76  \pm 0.04$ & $4.47  \pm 0.03$ & $0.55  \pm 0.04$ \\
		6.0       & $4.47  \pm 0.00$             & $0.71  \pm 0.05$              & $4.47  \pm 0.00$             & $0.72  \pm 0.06$ & $4.46  \pm 0.03$ & $0.50  \pm 0.04$ \\
		9.0       & $4.47  \pm 0.01$             & $0.62  \pm 0.02$              & $4.47  \pm 0.00$             & $0.63  \pm 0.03$ & $4.47  \pm 0.01$ & $0.45  \pm 0.02$ \\
		no        & $4.48  \pm 0.01$             & $0.67  \pm 0.04$              & $4.47  \pm 0.00$             & $0.69  \pm 0.04$ & $4.46  \pm 0.02$ & $0.48  \pm 0.02$ \\
		\bottomrule
	\end{tabular}
\end{table}

\FloatBarrier

\section{Architectural and Training details}
\label{appendix:architectural_details}

We make several small modifications to the standard ResNet architecture in order to make our theoretical analysis easier to verify.
Firstly, we move BatchNorm from the residual branches and place it between residual blocks instead, as we found that regularising BatchNorm to have Lipschitz < 1, while fairly easy to implement \citep{gouk2018regularisation}, was difficult to train because this significantly changes the dynamics of how BatchNorm is supposed to behave.
It is common to use a strided 1x1 convolution on the skip connection when the residual branch also has a downsampling operation.
This goes against our assumption that the residual connection is an identity.
However, a residual connection which simply performs downsampling with BlurPool is an identity on the low frequency components of the input.
Therefore, in order to make the structure of the network conform to our analysis, we replace the 1x1 convolution on the skip connection with BlurPool operation.
To handle the increasing number of channels, we use zero padding.
This operations is isometric (that is, $||x - y|| = ||f(x) - f(y)|| \forall x, y$), and so all proofs apply.
These changes make it simpler to verify our theoretical conditions by comparing distances between feature maps before and after the residual connection.
We use BlurPool in all convolution layers that perform downsampling, as described in \citet{zhang2019making}.

We use the standard train/test splits for all datasets used.
We use SGD with momentum 0.9 and a learning rate schedule to train all models; the initial learning rate is 0.1.
On CIFAR10, we divide this by 10 at epochs 150 and 250 and train for 350 epochs in total.
For MNIST and FashionMNIST, we train for 200 epochs and divide the learning rate by 5 at epochs 60, 120 and 160.
These are following the schedules used in \citet{mukhoti2021deterministic}.
We do this to produce trained models which can reasonably be said to be representative; though our accuracy and AUROC are not SotA, they are comparable to the results achieved in the literature.

We trained our models on an internal cluster, using a mix of GeForce 1080's, 2080's and Titan RTX's.
We did not record the exact amount of compute time used, but it was relatively modest - on the order of a few GPU days to generate all listed results.

\end{document}